\documentclass{article}

\usepackage[preprint]{neurips_2023}

\usepackage[utf8]{inputenc} 
\usepackage[T1]{fontenc}    
\usepackage{hyperref}       
\usepackage{url}            
\usepackage{booktabs}       
\usepackage{amsfonts}       
\usepackage{nicefrac}       
\usepackage{microtype}      
\usepackage{xcolor}         
\usepackage{subcaption}
\usepackage{caption}
\usepackage{wrapfig}

\usepackage{amsmath}
\usepackage{amssymb}
\usepackage{mathtools}
\usepackage{amsthm}
\usepackage{tikz}
\usetikzlibrary{arrows.meta}
\usetikzlibrary{shapes.geometric}
\usetikzlibrary{shapes,backgrounds}
\usepackage{bbm}
\usepackage{natbib}
\setcitestyle{numbers,square,comma}
\usepackage{xcolor}
\definecolor{ntnublue}{RGB}{0,80,158}
\definecolor{ntnublue_light}{RGB}{109,183,102}

\theoremstyle{plain}
\newtheorem*{theorem*}{Theorem}
\newtheorem{theorem}{Theorem}[section]
\newtheorem{proposition}[theorem]{Proposition}

\theoremstyle{definition}
\newtheorem{definition}[theorem]{Definition}

\theoremstyle{remark}

\newcommand{\R}{\mathbb{R}}

\newcommand{\V}{\text{Var}}
\newcommand{\hlf}{\frac{1}{2}}
\newcommand{\one}{\mathbbm{1}}
\newcommand{\sqa}{\frac{\sqrt{21}}{98}}
\newcommand{\sqb}{\frac{7\sqrt{21}}{128}}

\newcommand{\sqd}{\frac{\sqrt{21}}{14}}
\newcommand{\fa}{\frac{1}{14}}

\title{Learning Dynamical Systems from Noisy Data with Inverse-Explicit Integrators}

\author{%
  Håkon Noren$^{\dagger}$ \quad Sølve Eidnes$^{\ddagger}$ \quad Elena Celledoni$^{\dagger}$  \\
$^\dagger$ Norwegian University of Science and Technology \quad $^\ddagger$ SINTEF Digital\\
\texttt{\{hakon.noren,elena.celledoni\}@ntnu.no}\\
\texttt{\{solve.eidnes\}@sintef.no}\\
}

\begin{document}

\maketitle

\begin{abstract}

  We introduce the mean inverse integrator (MII), a novel approach to increase the accuracy when training neural networks to approximate vector fields of dynamical systems from noisy data. This method can be used to average multiple trajectories obtained by numerical integrators such as Runge--Kutta methods. We show that the class of mono-implicit Runge--Kutta methods (MIRK) has particular advantages when used in connection with MII. When training vector field approximations, explicit expressions for the loss functions are obtained when inserting the training data in the MIRK formulae, unlocking symmetric and high-order integrators that would otherwise be implicit for initial value problems. The combined approach of applying MIRK within MII yields a significantly lower error compared to the plain use of the numerical integrator without averaging the trajectories. This is demonstrated with experiments using data from several (chaotic) Hamiltonian systems. Additionally, we perform a sensitivity analysis of the loss functions under normally distributed perturbations, supporting the favorable performance of MII.

\end{abstract}

\section{Introduction}
\label{introduction}

Recently, many deep learning methodologies have been introduced to increase the efficiency and quality of scientific computations \cite{raissi2019physics,rackauckas2020universal,chen2018neural,li2020fourier}. In physics-informed machine learning, deep neural networks are purposely built to enforce physical laws. As an example, Hamiltonian neural networks (HNNs) \cite{HamiltonianNeuralNetworks} aim at learning the Hamiltonian function from temporal observations. The Hamiltonian formalism was derived from classical mechanics for modeling a wide variety of physical systems. The temporal evolution of such systems is fully determined when the Hamiltonian function is known, and it is characterized by geometric properties such as the preservation of energy, the symplectic structure and the time-reversal symmetry of the flow \cite{GoldsteinHerbert2014Cm,leimkuhler_reich_2005}.

Numerical integrators that compute solutions preserving such properties are studied in the field of geometric numerical integration \cite{leimkuhler_reich_2005,hairer2006geometric}. Thus, deep learning, classical mechanics and geometric numerical integration are all relevant to the development of HNNs. In this work, we try to identify the optimal strategy for using numerical integrators when constructing loss functions for HNNs that are trained on noisy and sparse data.

Generally, we aim at learning autonomous systems of first-order ordinary differential equations (ODE)
\begin{equation}\label{eq:ode}
\frac{d}{dt} y = f(y(t)), \quad y : [0,T] \rightarrow \R^n.
\end{equation}
In the traditional setting, solving an initial value problem (IVP) means computing approximated solutions $y_n \approx y(t_n)$ when the vector field $f(y)$ and an initial value $y(t_0) = y_0$ are known. The focus of our study is the corresponding inverse problem; assuming knowledge of multiple noisy samples of the solution, $S_N = \{\tilde y_n\}_{n=0}^N$, the aim is to approximate the vector field $f$ with a neural network model $f_\theta$. 
We will assume that the observations originate from a (canonical) Hamiltonian system, with a Hamiltonian $H : \R^{2d} \rightarrow \R$, where the vector field is given by
\begin{equation}
  \label{hamiltonian system}
    f(y) = J\nabla H(y(t)), \quad J := \begin{bmatrix}
      0 & I\\
      -I & 0
    \end{bmatrix} \in \R^{2d \times 2d}.
  \end{equation}
This allows for learning the Hamiltonian function directly by setting $f_{\theta}(y) = J\nabla H_{\theta}(y)$, as proposed initially in \cite{HamiltonianNeuralNetworks}. 

Recently, many works highlight the benefit of using symplectic integrators when learning Hamiltonian neural networks \cite{offen2022symplectic,Chen2020Symplectic,DeepHamiltoniannetworksbasedonsymplecticintegrators,SymplecticLearningforHamiltonianNeuralNetworks}. Here, we study what happens if, instead of using symplectic methods, efficient and higher-order MIRK methods are applied for inverse problems. We develop different approaches and apply them to learn highly oscillatory and chaotic dynamical systems from noisy data. The methods are general, they are not limited to separable Hamiltonian systems, and could indeed be used to learn any first-order ODE. However, we focus our study on Hamiltonian systems, in order to build on the latest research on HNNs. Specifically, we compare our methods to the use of symplectic integrators to train Hamiltonian neural networks. Our contributions can be summarized as follows: 

\begin{itemize}
  \item We introduce the \textbf{mean inverse integrator} (MII), which efficiently averages trajectories of MIRK methods in order to increase accuracy when learning vector fields from noisy data (Definition \ref{mii_equation}).
  \item We present an \textbf{analysis of the sensitivity} of the loss function to perturbations giving insight into when the MII method yields improvement over a standard one-step scheme (Theorem \ref{thm:propagation of noise mirk}).
  \item We show that \textbf{symplectic MIRK} methods have at most order  $p = 2$ (Theorem \ref{max_ord_symp}). Particularly, the second-order implicit midpoint method is the symplectic MIRK method with minimal number of stages.
\end{itemize}
Finally, numerical experiments on several Hamiltonian systems benchmark MII against one-step training and symplectic recurrent neural networks (SRNN) \cite{Chen2020Symplectic}, which rely on the Störmer--Verlet integrator. The structural difference between these three approached is presented in Figure \ref{training_structure}. Additionally, we demonstrate that substituting Störmer--Verlet with the classic Runge--Kutta method (RK$4$) in the SRNN framework yields a significant reduction in error and allows accurate learning of non-separable Hamiltonian systems.

\section{Related work}
Hamiltonian neural networks was introduced in \cite{HamiltonianNeuralNetworks}.  The numerical integration of Hamiltonian ODEs and the preservation of the symplectic structure of the ODE flow under numerical discretization have been widely studied over several decades \cite{hairer2006geometric, leimkuhler_reich_2005}. The symplecticity property is key and could inform the neural network architecture \cite{SympNetsIntrinsicstructure-preservingsymplecticnetworksforidentifyingHamiltoniansystems} or guide the choice of numerical integrator, yielding a theoretical guarantee that the learning target is actually a (modified) Hamiltonian vector field \cite{Inversemodifieddifferentialequationsfordiscoveryofdynamics,offen2022symplectic}, building on the backward error analysis framework \cite{hairer2006geometric}. Discrete gradients is an approach to numerical integration that guarantees exact preservation of the (learned) Hamiltonian, and an algorithm for training Hamiltonian neural networks using discrete gradient integrators is developed in \cite{DeepEnergy-BasedModelingofDiscrete-TimePhysics} and extended to higher order in \cite{eidnes2022order}.

Since we for the inverse problem want to approximate the time-derivative of the solution, $f$, using only $\tilde y_n$, we need to use a numerical integrator when specifying the neural network loss function. For learning dynamical systems from data, explicit methods such as RK$4$ are much used \cite{HamiltonianNeuralNetworks,celledoni23lho,sanchez-gonzalezHamiltonianGraphNetworks2019}. However, explicit methods cannot in general preserve time-symmetry or symplecticity, and they often have worse stability properties compared to implicit methods \cite{wanner1996solving}. Assuming that the underlying Hamiltonian is separable allows for explicit integration with the symplectic Störmer--Verlet method, which is exploited in \cite{Chen2020Symplectic,liang_stiffness-aware_2022}. Symplecticity could be achieved without the limiting assumption of separability by training using the implicit midpoint method \cite{SymplecticLearningforHamiltonianNeuralNetworks}. As pointed out in \cite{SymplecticLearningforHamiltonianNeuralNetworks}, this integrator could be turned into an explicit method in training by inserting sequential training data $\tilde y_n$ and $\tilde y_{n+1}$. In fact, the MIRK class \cite{cash1975class,burrage1994order} contains all Runge--Kutta (RK) methods (including the midpoint method) that could be turned into explicit schemes when inserting the training data. This is exploited in \cite{noren2023learning}, where high-order MIRK methods are used to train HNNs, achieving accurate interpolation and extrapolation of a single trajectory with large step size, few samples and assuming zero noise. 

The assumption of noise-free data limits the potential of learning from physical measurements or applications on data sets from industry. This issue is addressed in \cite{Chen2020Symplectic}, presenting symplectic recurrent neural networks (SRNN). Here, Störmer--Verlet is used to integrate multiple steps and is combined with initial state optimization (ISO) before computing the loss. ISO is applied after training $f_{\theta}$ a given number of epochs and aims at finding the optimal initial value $\hat y_0$, such that the distance to the subsequent observed points $\tilde y_1,\dots,\tilde y_N$ is minimized when integrating over $f_{\theta}$. While \cite{Chen2020Symplectic} is limited by only considering separable systems, \cite{sharmaBayesianIdentificationNonseparable2022} aims at identifying the optimal combination of third-order polynomial basis functions to approximate a cubic non-separable Hamiltonian from noisy data, using a Bayesian framework.

\section{Background on numerical integration}
Some necessary and fundamental concepts on numerical integration and the geometry of Hamiltonian systems are presented below to inform the discussion on which integrators to use in inverse problems. Further details could be found in Appendix \ref{more_on_num_int}.

\textbf{Fundamental concepts:}
An important subclass of the general first-order ODEs \eqref{eq:ode} is the class of Hamiltonian systems, as given by \eqref{hamiltonian system}. Often, the solution is partitioned into the coordinates $y(t) = [q(t),p(t)]^T$, with $q(t),p(t) \in \R^d$. A separable Hamiltonian system is one where the Hamiltonian could be written as the sum of two scalar functions, often representing the kinetic and potential energy, that depends only on $q$ and $p$ respectively, this means we have $H(q,p) = H_1(q) + H_2(p)$.

The $h$ flow of an ODE is a map $\varphi_{h,f}:\mathbb{R}^n\rightarrow \mathbb{R}^n$ sending an initial value $y(t_0)$ to the solution of the ODE at time $t_0 + h$, given by
$
  \varphi_{h,f}(y(t_0)) := y(t_{0} + h)
$.
A numerical integration method $\Phi_{h,f}:\mathbb{R}^n\rightarrow \mathbb{R}^n$ is a map approximating the exact flow of the ODE, 
so that
\[
y(t_{1})\approx y_{1} = \Phi_{h,f}(y_0).
\]
Here, $y(t_n)$ represents the exact solution and we denote with $y_n$ the approximation at time $t_n = t_0 + n h$. It should be noted that the flow map satisfies the following group property:
\begin{equation}
  \label{group_flow}
  \varphi_{h_1,f} \circ \varphi_{h_2,f}\big (y(t_0 )\big) = \varphi_{h_1,f} \big (y(t_0 + h_2 )\big) = \varphi_{h_1 + h_2,f}(y(t_0)).
\end{equation}
In other words, a composition of two flows with step sizes $h_1,h_2$ is equivalent to the flow map over $f$ with step size $h_1 + h_2$. This property is not shared by numerical integrators for general vector fields. The order of a numerical integrator $\Phi_{h,f}$ characterizes how the error after one step depends on the step size $h$ and is given by the integer $p$ such that the following holds:
\[
\| y_{1} - y(t_{0} + h) \| = \| \Phi_{h,f}(y_0) - \varphi_{h,f}(y(t_0)) \| = \mathcal O (h^{p+1}).
\]

\textbf{Mono-implicit Runge--Kutta methods:} Given vectors $b,v\in\R^s$ and a strictly lower triangular matrix $D \in \R^{s\times s}$, a MIRK method is a Runge--Kutta method where $A = D + vb^T$ \cite{vanBokhoven1980efficient, Cash1982mono} and we assume that $[A]_{ij} = a_{ij}$ is the stage-coefficient matrix. This implies that the MIRK method can be written on the form
\begin{equation}
\begin{aligned}
    \label{mirk-stages}
        y_{n+1} &= y_n + h\sum_{i=1}^s b_i k_i,   \\
    k_i &= f\big (y_n + v_i(y_{n+1} - y_n) + h\sum_{j=1}^s d_{ij} k_j \big).
  \end{aligned} 
\end{equation}
Specific MIRK methods and further details on Runge--Kutta schemes is discussed in Appendix \ref{more_on_mirk}.

\textbf{Symplectic methods:} The flow map of a Hamiltonian system is symplectic, meaning that it's Jacobian $\Upsilon_{\varphi} := \frac{\partial}{\partial y} \varphi_{h,f}(y)$ satisfies $\Upsilon_{\varphi} ^T J \Upsilon_{\varphi} = J$, where $J$ is the same matrix as in \eqref{hamiltonian system}. As explained in \cite[Ch. VI.2]{hairer2006geometric}, this is equivalent to the preservation of a projected area in the phase space of $[q,p]^T$. Similarly, a numerical integrator is symplectic if its Jacobian $\Upsilon_{\Phi} := \frac{\partial}{\partial y_n} \Phi_{h,f}(y_n)$ satisfies $\Upsilon_{\Phi} ^T J \Upsilon_{\Phi} = J$. It is possible to prove \cite[Ch. VI.4]{hairer2006geometric} that a Runge--Kutta method is symplectic if and only if the coefficients satisfy
\begin{equation}
b_i a_{ij} + b_j a_{ji} - b_i b_j = 0, \quad i,j = 1,\dots,s.
\label{symplectic_condition_rk}
\end{equation}

 \section{Numerical integration schemes for solving inverse problems}\label{sec:inverse_problems}

  We will now consider different ways to use numerical integrators when training Hamiltonian neural networks and present important properties of MIRK methods, a key component of the MII that is presented in Chapter \ref{sec:mii_noise}.

\textbf{Inverse ODE problems in Hamiltonian form:}
   We assume to have potentially noisy samples $S_N = \{\tilde y\}_{n=0}^N$ of the solution of an ODE with vector field $f$. The inverse problem can be formulated as the following optimization problem:
 \begin{equation}
   \begin{aligned}
   \operatorname*{arg\,min}_\theta  \sum_{n=0}^{N-1} \bigg \| \tilde y_{n+1} - \Phi_{h,f_{\theta}}(\tilde y_n)   \bigg \|,
   \label{inverse_problem}
   \end{aligned}
 \end{equation}
 where $f_{\theta} = J\nabla H_{\theta}$ is a neural network approximation with parameters $\theta$ of a Hamiltonian vector field $f$, and $\Phi_{h,f_{\theta}}$ is a one-step integration method with step length $h$.
  \begin{wrapfigure}{r}{0.45\textwidth}
    \centering
    \small
    
    \def\RK{(0,0) circle (2cm)}
    \def\SMIRK{(1.5,-0.5) circle (1.5cm)}
    \def\MIRK{(30:-0.7cm) circle (1.3cm)}
    \def\ERK{(180:1.05cm) circle (0.7cm)}
    \def\symplectic{(-5:1cm) ellipse (1.2cm and 0.9cm) }

    \begin{tikzpicture}
        \begin{scope}[shift={(3cm,-5cm)}, fill opacity=0.4]
    
            \draw \RK node[above] {};
            \draw[fill=gray] \ERK node {ERK};
  
            \clip \RK;
            \draw[rotate = 45,fill=orange] \symplectic;
    
            \draw[fill=ntnublue] \MIRK;
            \clip \RK;
            \draw[fill=ntnublue] \SMIRK;
    
        \end{scope}
            \begin{scope}[shift={(3cm,-5cm)}]
            \draw node (ERK) at (180:1.05cm) {ERK};
            \draw node (RK) at (-0.2cm,1.5cm) {RK};
            \draw node (MIRK) at (-125:1.2cm) {MIRK};
            \draw node (SRK) at (55:1.45cm) {SympRK};
            \draw node (SymRK) at (-25:1.4cm) {SymRK};

        \node (mirk4b) at (2,-2.2) {I. Euler, MIRK3, MIRK5};
        \node (rk4b) at (-1.1,2.2) {E. Euler, RK4};
        \node (gl4a) at (1.2,0.6) {};
        \node (gl4b) at (3.2,0.9) {GL4, GL6};
        \node (mirk6a) at (0.2,-0.8) {};
        \node (mirk6b) at (3.6,-1.2) {MIRK4, MIRK6};
        
        \node (midtpoint_a) at (0.2,0.1) {};
        \node (midtpoint_b) at (3.3,-0.2) {Midpoint};
    
      \draw[-stealth,line width=0.4pt] (MIRK.south) to [out=-30,in=165] (mirk4b.west);
    \draw[-stealth,line width=0.4pt] (ERK.west) to [out=110,in=-80] (rk4b);
    \draw[-stealth,line width=0.4pt] (gl4a.east) to [out=0,in=190] (gl4b.west);
    \draw[-stealth,line width=0.4pt] (mirk6a.east) to [out=0,in=190] (mirk6b.west);
    \draw[-stealth,line width=0.4pt] (midtpoint_a.east) to [out=0,in=190] (midtpoint_b.west);
    
         \end{scope}
    \end{tikzpicture}
  
     \caption{Venn diagram of Runge--Kutta (RK) subclasses: explicit RK (ERK), symplectic RK (SympRK), mono-implicit RK (MIRK) and symmetric RK (SymRK).}
    \label{fig:rkclasses}
\end{wrapfigure}
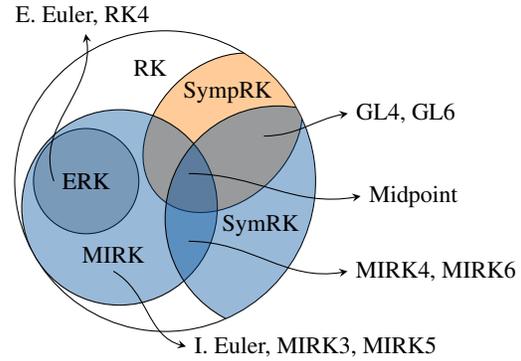
 In the setting of inverse ODE problems, the availability of sequential points $S_N$ could be exploited when a numerical method is used to form interpolation conditions, for $f_{\theta} \approx f$ for each $n$ in the optimization problem \eqref{inverse_problem}. For example, $\tilde y_{n}$ and $\tilde y_{n+1}$ could be inserted in the implicit midpoint method, turning a method that is implicit for IVPs into an explicit method for inverse problems:
\begin{equation}\label{eq:midpoint}
 \Phi_{h,f_{\theta}}(\tilde y_n,\tilde y_{n+1}) = \tilde y_n + hf_{\theta}\big( \frac{\tilde y_n +  \tilde y_{n+1}}{2} \big).
  \end{equation} 
 We denote this as the inverse injection, which defines an inverse explicit property for numerical integrators.
 \begin{definition}[Inverse injection]
 Assume that $\tilde y_n,\tilde y_{n+1} \in S_N$. Let the \textit{inverse injection} for the integrator  $\Phi_{h,f}(y_n,y_{n+1})$ be given by the substitution $(\tilde y_n,\tilde y_{n+1}) \rightarrow (y_n,y_{n+1})$ such that
 $$
 \hat y_{n+1} = \Phi_{h,f}(\tilde y_n,\tilde y_{n+1}).
 $$
 \end{definition}
  \begin{definition}[Inverse explicit]
 A numerical one-step method $\Phi$ is called \textit{inverse explicit} if it is explicit under the inverse injection. 
 \end{definition}
 This procedure is utilized successfully by several authors when learning dynamical systems from data, see e.g.\ \cite{ SymplecticLearningforHamiltonianNeuralNetworks, PortHam_eidnes}. However, this work is the first attempt at systematically exploring numerical integrators under the inverse injection, by identifying the MIRK methods as the class consisting of inverse explicit Runge--Kutta methods.
 \begin{proposition}\label{lemma_inv_exp_mirk}
   MIRK-methods are inverse explicit.
 \end{proposition}
 
 \begin{proof}
   Since the matrix $D$ in \eqref{mirk-stages} is strictly lower triangular, the stages are given by 
   \begin{align*}
         k_1 &= f (y_n + v_i(y_{n+1} - y_n) )\\
         k_2 &= f (y_n + v_i(y_{n+1} - y_n) + h d_{21} k_1 )\\
         &\vdots\\ 
  k_s &= f (y_n + v_i(y_{n+1} - y_n) + h\sum_{j=1}^{s-1} d_{sj} k_j )
     \end{align*}
     meaning that if $y_n$ and $y_{n+1}$ are known, all stages, and thus the next step $\hat y_{n+1} = y_n + h\sum_{i=1}^s b_i k_i$, could be computed explicitly.
 \end{proof}
 Because of their explicit nature when applied to inverse ODE problems, MIRK methods are an attractive alternative to explicit Runge--Kutta methods; in contrast to explicit RK methods, they can be symplectic or symmetric, or both, without requiring the solution of systems of nonlinear equations, even when the Hamiltonian is non-separable. Figure \ref{fig:rkclasses} illustrates the relation between various subclasses and the specific methods are described in Table \ref{integrator table} in Appendix \ref{more_on_num_int}. 
 
 In addition, for $s$-stage MIRK methods, it is possible to construct methods of order $p = s+1$ \cite{burrage1994order}. This is in general higher order than what is possible to obtain with $s$-stage explicit Runge--Kutta methods. Further, computational gains could also be made by reusing evaluations of the vector field between multiple steps, which using MIRK methods allow for, as explained in Appendix \ref{appendix:comp_cost}. The dependency structure on the data $S_N$ of explicit RK (ERK) methods, MIRK methods and the SRNN method \cite{Chen2020Symplectic} is illustrated in Figure \ref{training_structure}.

 \textbf{Maximal order of symplectic MIRK methods:}
 From the preceding discussion, it is clear that symplectic MIRK methods are of interest when learning Hamiltonian systems from data, since they combine computational efficiency 
 with the ability to preserve useful, geometric properties. Indeed, symplectic integrators in the training of HNNs have been considered in \cite{offen2022symplectic,Chen2020Symplectic,DeepHamiltoniannetworksbasedonsymplecticintegrators,SymplecticLearningforHamiltonianNeuralNetworks, SympNetsIntrinsicstructure-preservingsymplecticnetworksforidentifyingHamiltoniansystems}. The subclass of symplectic MIRK methods is represented by the middle, dark blue field in the Venn diagram of Figure \ref{fig:rkclasses}. The next result gives an order barrier for symplectic MIRK methods that was, to the best of our knowledge, not known up to this point.

 \begin{theorem}\label{max_ord_symp}
   The maximum order of a symplectic MIRK method is $p=2$.
 \end{theorem}\vspace{-10pt}
 \begin{proof}
   This is a shortened version of the full proof, which can be found in Appendix \ref{proof_max_order}.
 A MIRK method is a Runge--Kutta method with coefficients $a_{ij} = d_{ij} + v_ib_j$. Requiring $d_{ij}, b_i$ and $v_i$ to satisfy the symplecticity conditions of \eqref{symplectic_condition_rk} in addition to $D$ being strictly lower triangular, yields the following restrictions
 \begin{equation}
 \begin{aligned}
 b_i d_{ij} + b_i b_j (v_j + v_i - 1) &= 0,  &&\text{if} \; i \neq j,\\
  b_i = 0 \;\; \text{or} \;\; v_i &= \frac{1}{2}, &&\text{if} \; i = j,\\
 d_{ij} &= 0, &&\text{if} \; i>j.
 \end{aligned}
 \end{equation}
 These restrictions result in an RK method that could be reduced to choosing a coefficient vector $b\in\R^s$ and choosing stages on the form
 $
     k_i = f\big(y_n + \frac{h}{2}\sum_{j}^s b_jk_j\big)
 $
 for $i = 1,\dots,s$. It is then trivial to check that this method can only be of up to order $p=2$. Note that for $s=1$ and $b_1 = 1$ we get the midpoint method.
 \end{proof}

 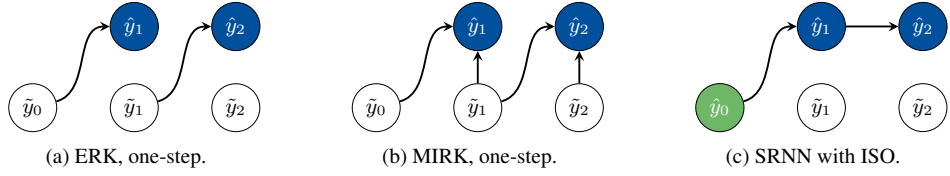
\begin{figure}
\centering
\scalebox{0.9}{
\begin{subfigure}[b!]{0.25\textwidth}
\small
\centering
\begin{tikzpicture}[
  main/.style = {draw, circle, minimum size=5mm},
  main_r2/.style = {draw, circle, minimum size=5mm, text = white, fill=ntnublue},
  ] 
  \small
  \node[main] (0) at (0,0) {$\tilde y_0$}; 
  \node[main] (1) at (1.5,0) {$\tilde y_1$};
  \node[main_r2] (h1) at (1.5,1.2) {$\hat y_1$};
  \node[main] (2) at (3,0) {$\tilde y_2$};
  \node[main_r2] (h2) at (3,1.2) {$\hat y_2$};
  \draw[-stealth,line width=0.8pt] (0) to  [out=15,in=180]  (h1);
  \draw[-stealth,line width=0.8pt] (1) to  [out=15,in=180]  (h2);
\end{tikzpicture} 
\caption{ERK, one-step.}
\end{subfigure}
\hspace{4em}
\begin{subfigure}[b!]{0.25\textwidth}
\centering
\begin{tikzpicture}[
  main/.style = {draw, circle, minimum size=5mm},
  main_r2/.style = {draw, circle, minimum size=5mm, text = white, fill=ntnublue},
  ] 
  \small
  \node[main] (0) at (0,0) {$\tilde y_0$}; 
  \node[main] (1) at (1.5,0) {$\tilde y_1$};
  \node[main_r2] (h1) at (1.5,1.2) {$\hat y_1$};
  \node[main] (2) at (3,0) {$\tilde y_2$};
  \node[main_r2] (h2) at (3,1.2) {$\hat y_2$};
  \draw[-stealth,line width=0.8pt] (0) to  [out=15,in=180]  (h1);
  \draw[-stealth,line width=0.8pt] (1) to   (h1);
  \draw[-stealth,line width=0.8pt] (1) to  [out=15,in=180]  (h2);
  \draw[-stealth,line width=0.8pt] (2) to   (h2);
\end{tikzpicture} 
\caption{MIRK, one-step.}
\end{subfigure}
\hspace{4em}
\begin{subfigure}[b!]{0.25\textwidth}
\begin{tikzpicture}[
  main/.style = {draw, circle, minimum size=5mm},
  main_r2/.style = {draw, circle, minimum size=5mm, text = white, fill=ntnublue},
    main_r3/.style = {draw, circle, minimum size=5mm, text = white, fill=ntnublue_light}
  ] 
  \small
  \node[main_r3] (0) at (0,0) {$\hat y_0$}; 
  \node[main] (1) at (1.5,0) {$\tilde y_1$};
  \node[main_r2] (h1) at (1.5,1.2) {$\hat y_1$};
  \node[main] (2) at (3,0) {$\tilde y_2$};
  \node[main_r2] (h2) at (3,1.2) {$\hat y_2$};
  \draw[-stealth,line width=0.8pt] (0) to  [out=15,in=180]  (h1);
  \draw[-stealth,line width=0.8pt] (h1) to  (h2);
\end{tikzpicture} 
\caption{SRNN with ISO.}
\end{subfigure}
}

 
\caption{Differences of observation dependency, assuming $N=2$ for explicit and mono-implicit one-step training, and explicit multi-step training with initial state optimization (green node $\hat y_0$). }
\label{training_structure}
\vspace{-1em}
\end{figure}

 \textbf{Numerical integrators outside the RK class:}
 While this paper is mainly concerned with MIRK methods, several other types of numerical integrators could be of interest for inverse problems. \textit{Partitioned Runge--Kutta methods} are an extension and not a subclass of RK methods, and can be symplectic and symmetric, while also being explicit for separable Hamiltonian systems.
 The Störmer--Verlet integrator of order $p=2$ is one example. Higher order methods of this type are derived in \cite{yoshida1990construction} and used for learning Hamiltonian systems in \cite{desai2021variational,dipietro2020sparse}. \textit{Discrete gradient methods} \cite{quispel1996discrete,mclachlan1999geometric} are inverse explicit and well suited to train Hamiltonian neural networks using a modified automatic differentiation algorithm \cite{DeepEnergy-BasedModelingofDiscrete-TimePhysics}. This method could be extended to higher order methods as shown in \cite{eidnes2022order}. In contrast to symplectic methods, discrete gradient methods preserve the Hamiltonian exactly up to machine precision. A third option is \textit{elementary differential Runge--Kutta methods} \cite{uria1995metodos}, where for instance \cite{chartier_numerical_2007} show how to use backward error analysis to construct higher order methods from modifications to the midpoint method. This topic is discussed further in Appendix \ref{appendix:higher_order_inv_explicit}, where we also present a novel, symmetric discrete gradient method of order $p=4$.

 \section{Mean inverse integrator for handling noisy data}\label{sec:mii_noise}
 
 \textbf{Noisy ODE sample:}
 It is often the case that the samples $S_N$ are not exact measurements of the system, but are perturbed by noise. In this paper, we model the noise as independent, normally distributed perturbations
     \begin{equation}
         \tilde y_n = y(t_n) + \delta_{n}, \quad \delta_{n} \sim \mathcal N(0,\sigma^2I),
         \label{noisy_obs}
       \end{equation}
     where $\mathcal N(0,\sigma^2I)$ represents the multivariate normal distribution. 
      With this assumption, a standard result from statistics tells us that the variance of a sample-mean estimator with $N$ samples converges to zero at the rate of $\frac{1}{N}$. That is, assuming that we have $N$ samples $\tilde y_n^{(1)},\dots,\tilde y_n^{(N)}$, then 
     \[
     \V[\overline y_n] = \V\bigg[\frac{1}{N}\sum_{j=1}^N \tilde y_n^{(j)}\bigg] = \frac{\sigma^2}{N}.
     \]
Using the inverse injection with the midpoint method, the vector field is evaluated in the average of $\tilde y_n$ and $\tilde y_{n+1}$, reducing the variance of the perturbation by a factor of two, compared to evaluating the vector field in $\tilde y_n$, as is done in all explicit RK methods. Furthermore, considering the whole data trajectory $S_N$, multiple independent approximations to the same point $y(t_n)$ can enable an even more accurate estimate. This is demonstrated in the analysis presented in Theorem \ref{thm:propagation of noise mirk} and in Figure \ref{fig:propagation of noise}.
 
     \textbf{Averaging multiple trajectories:}
    In the inverse ODE problem, we assume that there exists an {\it exact} vector field $f$ whose flow interpolates the discrete trajectories $S_N$, and the flow of this vector field satisfies the group property \eqref{group_flow}.
   The numerical flow $\Phi_{h,f}$ for a method of order $p$ satisfies this property only up to an error $\mathcal{O}(h^{p+1})$ over one step. 
   In the presence of noisy data, compositions of one-step methods can be used to obtain multiple different approximations to the same point $y(t_n)$, by following the numerical flow from different nearby initial values $\tilde y_j, j\neq n$, and thus reduce the noise by averaging over these multiple approximations. Accumulation of the local truncation error is expected when relying on points further away from $t_n$. However, for sufficiently small step sizes $h$ compared to the size of the noise $\sigma$, one can expect increased accuracy when averaging over multiple noisy samples. 

   As an example, assume that we know the points $\{\tilde y_0,\tilde y_1,\tilde y_2,\tilde y_3\}$. Then $ y(t_2)$ can be approximated by computing the mean of the numerical flows $\Phi_{h,f}$ starting from different initial values:
 \begin{equation}
 \begin{split}
     \overline y_{2} &= \frac{1}{3}  \big ( \Phi_{h,f}(\tilde y_{1}) + \Phi_{h,f} \circ \Phi_{h,f}(\tilde y_{0}) + \Phi^*_{-h,f} (\tilde y_{3}) \big) \\
     &\approx \frac{1}{3}  \big (\tilde y_0 +\tilde y_1 +\tilde y_3 + h( \Psi_{0,1} + 2 \Psi_{1,2} -  \Psi_{2,3} )  \big),
 \end{split}
     \label{averaging_trajectories}
 \end{equation}
 where we by $\Phi^*$ mean the adjoint method of $\Phi$, as defined in \cite[Ch. V]{hairer2006geometric}, and we let $\Psi_{n,n+1}$ be the increment of an inverse-explicit numerical integrator, so that 
 \[
 \Phi_{h,f}(\tilde y_n,\tilde  y_{n+1}) = \tilde y_n + h \Psi_{n,n+1}.
 \]
 For example, for the midpoint method, we have that $ \Psi_{n,n+1} = f(\frac{\tilde y_n + \tilde y_{n+1}}{2})$. When stepping in negative time in \eqref{averaging_trajectories}, we use the adjoint method in order to minimize the number of vector field evaluations, also when non-symmetric methods are used (which implies that we always use e.g.\ $\Psi_{1,2}$ and not $\Psi_{2,1}$). Note that in order to derive the approximation in  \eqref{averaging_trajectories}, repeated use of the inverse injection allows the known points $\tilde y_n$  to form an explicit integration procedure, where the composition of integration steps are approximated by summation over increments $\Psi_{n,n+1}$. This approximation procedure is presented in greater detail in Appendix \ref{appendix:detail_of_mii}.

\textbf{Mean inverse integrator:}
The mean approximation over the whole trajectory $\overline y_{n}$, for $n = 0,\dots,N$, could be computed simultaneously, reusing multiple vector field evaluations in an efficient manner. This leads to what we call the mean inverse integrator. For example, when $N=3$ we get
\begin{equation*}
    \begin{bmatrix}
      \overline y_0\\\overline y_1\\\overline y_2\\\overline y_3\\
    \end{bmatrix}\! =
     \frac{1}{3} \!
      \begin{bmatrix}
        0&1&1&1\\
        1&0&1&1\\
        1&1&0&1\\
        1&1&1&0\\
      \end{bmatrix} \!\!\!
      \begin{bmatrix}
         \tilde y_0\\ \tilde  y_1\\ \tilde   y_2\\ \tilde y_3\\
      \end{bmatrix}
      +
     \frac{h}{3} \!
     \begin{bmatrix}
      -3&-2&-1\\
      1&-2&-1\\
      1&2&-1\\
      1&2&3
    \end{bmatrix} \!\!\! 
    \begin{bmatrix}
      \Psi_{0,1}\\ \Psi_{1,2}\\ \Psi_{2,3}
    \end{bmatrix},
    \label{mean}
  \end{equation*}
  and the same structure is illustrated in Figure \ref{mii_structure}.
  
  \begin{definition}[Mean inverse integrator]
  \label{mii_equation}
For a sample $S_N$ and an inverse-explicit integrator $\Psi_{n,n+1}$, the mean inverse integrator is given by
  \begin{align}
    \overline Y &= \frac{1}{N}\bigg ( U \tilde Y + h W \Psi    \bigg )
\end{align}
where $\tilde  Y := [\tilde  y_0, \dots, \tilde y_N]^T \in \R^{(N+1) \times m}$, 
    $\Psi := [\Psi_{0,1},\dots,\Psi_{N-1,N}]^T\in \R^{N\times m}$.

Finally, $U\in \R^{(N+1)\times (N+1)}$ and $W\in \R^{(N+1) \times N}$ are given by
  \begin{align*}
    [U]_{ij} &:= \begin{cases}
        0&\text{if} \quad  i = j\\
        1 &\text{else}
    \end{cases}
    \quad\quad\text{and}\quad\quad
    [W]_{ij} := \begin{cases}
        j - 1 - N &\text{if} \quad j \geq i\\
        j &\text{else}
    \end{cases}.
  \end{align*}
  \end{definition}
  By substituting the known vector field $f$ with a neural network $f_{\theta}$ and denoting the matrix containing vector field evaluations by $\Psi_{\theta}$ such that $\overline Y_{\theta} := \frac{1}{N} ( U \tilde Y + h W \Psi_{\theta} )$, we can formulate an analogue to the inverse problem \eqref{inverse_problem} by
  \begin{equation}
    \operatorname*{arg\,min}_\theta \big \|\tilde  Y - \overline Y_{\theta}   \big \|.
    \label{noisy_opt}
  \end{equation}

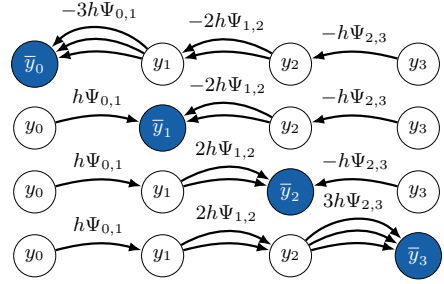
\begin{wrapfigure}{r}{0.4\textwidth}
\vspace{-10pt}
    \centering
\small
\scalebox{0.85}{
  \begin{tikzpicture}[
    main/.style = {draw, circle, minimum size=4mm},
    main_r2/.style = {draw, circle, minimum size=4mm, text = white, fill={rgb:white,1;ntnublue,10}},
    scale=1, 
    ] 
  
    \node[main_r2] (1) at (0,0) {$\overline y_0$}; 
    \node[main] (2) at (2,0) {$y_1$};
    \node[main] (3) at (4,0) {$y_2$};
    \node[main] (4) at (6,0) {$y_3$};
  
    \draw[{Latex[length=2mm]}-,line width=1pt] (3) to [out=15,in=165] node[above]{$ -h\Psi_{2,3} $}  (4);
    \draw[{Latex[length=2mm]}-,line width=1pt] (2) to [out=15,in=165] node[above]{$ $}  (3);
    \draw[{Latex[length=2mm]}-,line width=1pt] (2) to [out=30,in=150] node[above]{$ -2h\Psi_{1,2} $}  (3);
    \draw[{Latex[length=2mm]}-,line width=1pt]  (1) to [out=15,in=165] node[above]{$ $}  (2) ;
    \draw[{Latex[length=2mm]}-,line width=1pt]  (1) to [out=30,in=150] node[above]{$ $}  (2) ;
    \draw[{Latex[length=2mm]}-,line width=1pt]  (1) to [out=45,in=135] node[above]{$ -3h\Psi_{0,1} $}  (2) ;

    \node[main] (1) at (0,-1.0) {$y_0$}; 
    \node[main_r2] (2) at (2,-1.0) {$\overline y_1$};
    \node[main] (3) at (4,-1.0) {$y_2$};
    \node[main] (4) at (6,-1.0) {$y_3$};
  
    \draw[-{Latex[length=2mm]},line width=1pt] (1) to [out=15,in=165] node[above]{$ h\Psi_{0,1} $}  (2);
    \draw[{Latex[length=2mm]}-,line width=1pt] (2) to [out=15,in=165] node[above]{$ $}  (3);
    \draw[{Latex[length=2mm]}-,line width=1pt] (2) to [out=30,in=150] node[above]{$ -2h\Psi_{1,2} $}  (3);
    \draw[{Latex[length=2mm]}-,line width=1pt] (3) to [out=15,in=165] node[above]{$ -h\Psi_{2,3} $}  (4);
    
    \node[main] (1) at (0,-2.0) {$y_0$}; 
    \node[main] (2) at (2,-2.0) {$y_1$};
    \node[main_r2] (3) at (4,-2.0) {$\overline y_2$};
    \node[main] (4) at (6,-2.0) {$y_3$};
  
    \draw[-{Latex[length=2mm]},line width=1pt] (1) to [out=15,in=165] node[above]{$ h\Psi_{0,1} $}  (2);
    \draw[-{Latex[length=2mm]},line width=1pt] (2) to [out=15,in=165] node[above]{$ $}  (3);
    \draw[-{Latex[length=2mm]},line width=1pt] (2) to [out=30,in=150] node[above]{$ 2h\Psi_{1,2} $}  (3);
    \draw[{Latex[length=2mm]}-,line width=1pt] (3) to [out=15,in=165] node[above]{$ -h\Psi_{2,3} $}  (4);

    \node[main] (1) at (0,-3.0) {$y_0$}; 
    \node[main] (2) at (2,-3.0) {$y_1$};
    \node[main] (3) at (4,-3.0) {$y_2$};
    \node[main_r2] (4) at (6,-3.0) {$\overline y_3$};
  
    \draw[-{Latex[length=2mm]},line width=1pt] (1) to [out=15,in=165] node[above]{$ h\Psi_{0,1} $}  (2);
    \draw[-{Latex[length=2mm]},line width=1pt] (2) to [out=15,in=165] node[above]{$ $}  (3);
    \draw[-{Latex[length=2mm]},line width=1pt] (2) to [out=30,in=150] node[above]{$ 2h\Psi_{1,2} $}  (3);
    \draw[-{Latex[length=2mm]},line width=1pt]  (3) to [out=15,in=165] node[above]{$ $}  (4) ;
    \draw[-{Latex[length=2mm]},line width=1pt]  (3) to [out=30,in=150] node[above]{$ $}  (4) ;
    \draw[-{Latex[length=2mm]},line width=1pt]  (3) to [out=45,in=135] node[above]{$ 3h\Psi_{2,3} $}  (4) ;
  
  \end{tikzpicture} 
  }

\caption{Illustration of the structure of the mean inverse integrator for $N=3$.}
\label{mii_structure}
\vspace{-10pt}
\end{wrapfigure}

  \textbf{Analysis of sensitivity to noise:} Consider the optimization problems using integrators either as one-step methods or MII by \eqref{inverse_problem} resp.\ \eqref{noisy_opt}. We want to investigate how uncertainty in the data $\tilde y_n$ introduces uncertainty in the optimization problem. Assume, for the purpose of analysis, that the underlying vector field $f(y)$ is known. Let
\begin{align*}
  \mathcal T^{\text{OS}}_{n} &:= \tilde y_n - \Phi_{h,f}(\tilde y_{n-1}, \tilde y_n),\\
  \mathcal T^{\text{MII}}_{n} &:= \tilde y_n - [\overline Y]_n
\end{align*}
be the \textit{optimization target} or the expression one aims to minimize using a one-step method (OS) and the MII, where $\overline Y$ is given by Definition \ref{mii_equation}. For a matrix $A$ with eigenvalues $\lambda_i(A)$, the spectral radius is given by $\rho(A) := \max_{i} |\lambda_i(A)|$. An analytic expression that approximates $\rho(\mathcal T^{\text{OS}}_{n})$ and $\rho(\mathcal T^{\text{MII}}_{n})$ by linearization of $f$ for a general MIRK method is provided below. 
\begin{theorem}
  \label{thm:propagation of noise mirk}
  Let $S_N = \{\tilde y_n \}_{n=0}^N$ be a set of noisy samples, equidistant in time with step size $h$, with Gaussian perturbations as defined by \eqref{noisy_obs} with variance $\sigma^2$. Assume that a MIRK integrator $\Phi_{h,f}$ is used as a one-step method. Then the spectral radius is approximated by
      \begin{align}
           \rho_n^{\text{OS}} := \rho \bigg(  \V \big [ \mathcal T^{\text{OS}}_n  \big] \bigg )&\approx
        \sigma^2 \bigg \|2I + hb^T(\one - 2v)\big(f' \! + \! f'^T\big) + h^2 Q^{\text{OS}} \bigg \|_2, \label{eq:var_os}\\
      \rho_n^{\text{MII}} :=\rho \bigg( \V \big [ \mathcal T^{\text{MII}}_n \big]\bigg ) &\approx 
    \frac{\sigma^2}{N} \bigg \|  (1 + N)I + hP_{nn} +  \frac{h}{N}\sum_{\substack{j = 0 \\j\neq n}}^s P_{nj} + \frac{h^2}{N} Q^{\text{MII}}  \bigg \|_2,\label{eq:var_mii}
\end{align} 

  where $f' := f'(y_n)$ and $P_{nj}, Q^{\text{OS}}$ and $Q^{\text{MII}}$ (defined in \eqref{variance_matrices} in Appendix \ref{proof:propagation of noise mirk}) are matrices independent of the step size $h$.
\end{theorem}

\begin{wrapfigure}{r}{0.4\textwidth}
  \vspace{-10pt}
    \centering
      \includegraphics[width=0.4\textwidth]{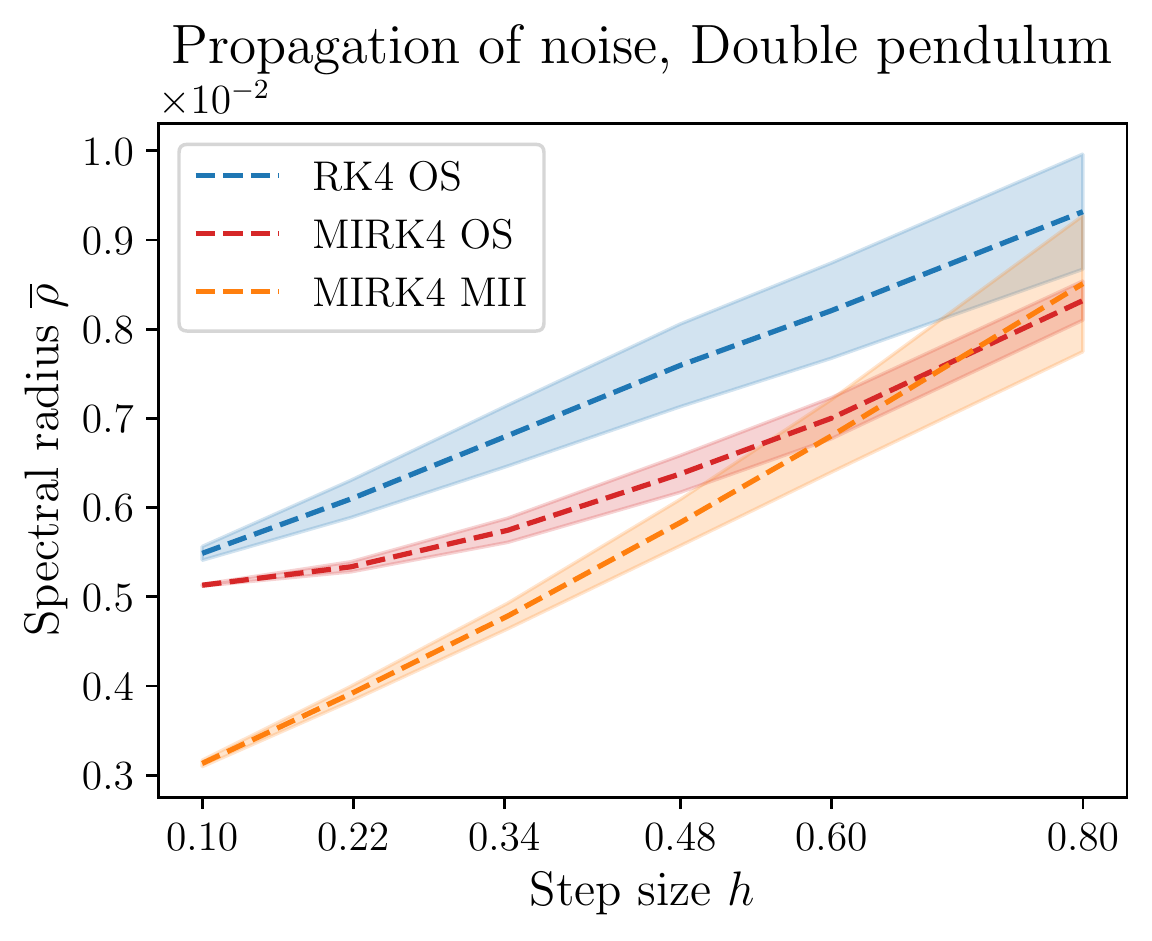}
    \caption{Average of $\overline \rho$ over $10$ trajectories. The shaded area represent one standard deviation.}\vspace{-10pt}
    \label{fig:propagation of noise}
  \end{wrapfigure}

The proof is found in Appendix \ref{proof:propagation of noise mirk}. Let $\alpha := b^T(\one - 2v)$ denote the coefficients of the first order term in $h$ of Equation \eqref{eq:var_os}. For any explicit RK method we have that $v = 0$ and since $b^T\one = 1$ (method of at least order one) we find that $\alpha_{\text{ERK}} = 1$. Considering the Butcher tableau of MIRK4 in Figure \ref{example_mirk} we find that $\alpha_{\text{MIRK}4} = 0$. Thus, as $h \rightarrow 0$ we would expect quadratic convergence of MIRK4 and linear convergence of RK4 for $\rho_n^{\text{OS}}$ to $2\sigma^2$. Considering MII \eqref{eq:var_mii} one would expect linear convergence for $\rho_n^{\text{MII}}$ to $\sigma^2$ if $N$ is large, as $h \rightarrow 0$.

A numerical approximation of $\rho^{\text{OS}}_n$ and $\rho^{\text{MII}}_n$ could be realized by a Monte-Carlo estimate. We compute the spectral radius $\hat \rho_n$ of the empirical covariance matrix of $\mathcal T^{\text{OS}}_{n}$ and $\mathcal T^{\text{MII}}_{n}$ by sampling $5 \cdot 10^3$ normally distributed perturbations $\delta_n$ with $\sigma^2 = 2.5\cdot 10^{-3}$ to each point $y_n$ in a trajectory of $N+1$ points and step size $h$. We then compute the trajectory average $\overline \rho = \frac{1}{N+1}\sum_{n=0}^{N}\hat \rho_n$, fix the end time $T = 2.4$, repeat the approximations for decreasing step sizes $h$ and increasing $N$ and compute the average of $\overline \rho$ for $10$ randomly sampled trajectories $S_N$ from the double pendulum system. The plot in Figure \ref{fig:propagation of noise} corresponds well with what one would expect from Theorem \ref{thm:propagation of noise mirk} and confirms that first MIRK (with $v \neq 0$) and secondly MII reduces the sensitivity to noise in the optimization target.

\section{Experiments}\label{sec:experiments}

\begin{wrapfigure}{r}{0.45\textwidth}
  \centering
  \includegraphics[trim=0 0 116 0,clip,width=0.45\textwidth]{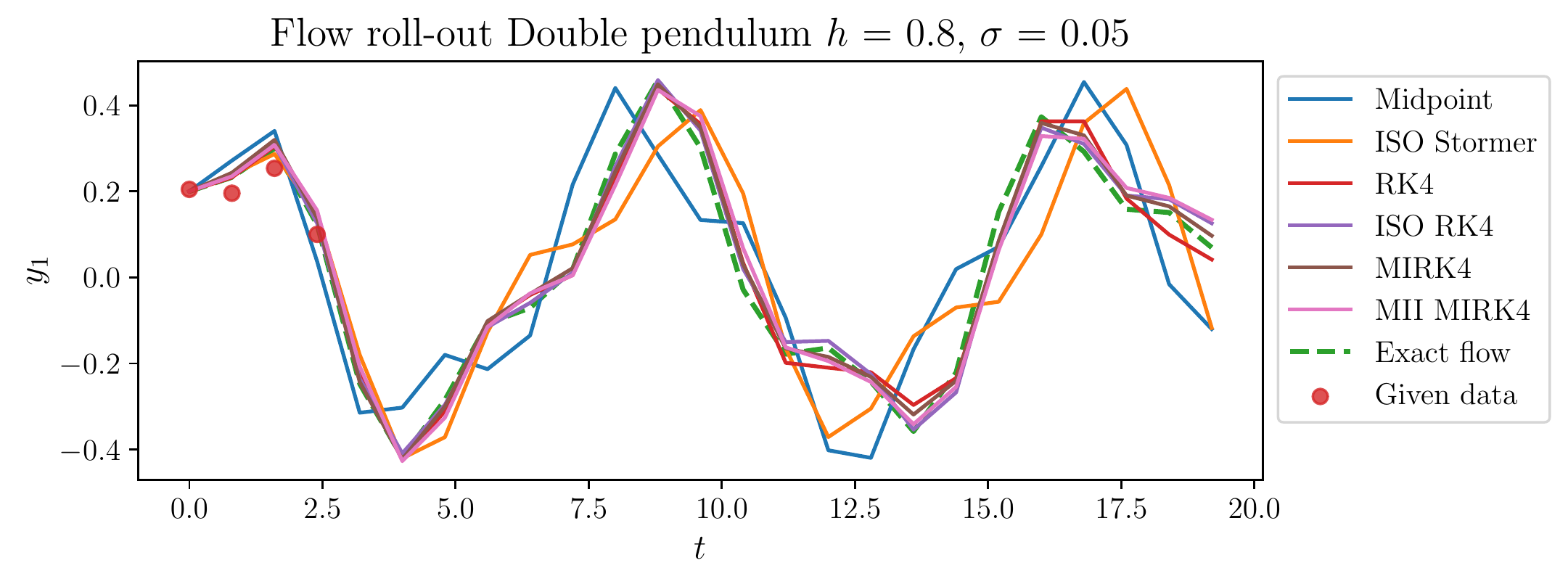}
  \includegraphics[trim=0 0 116 0,clip,width=0.45\textwidth]{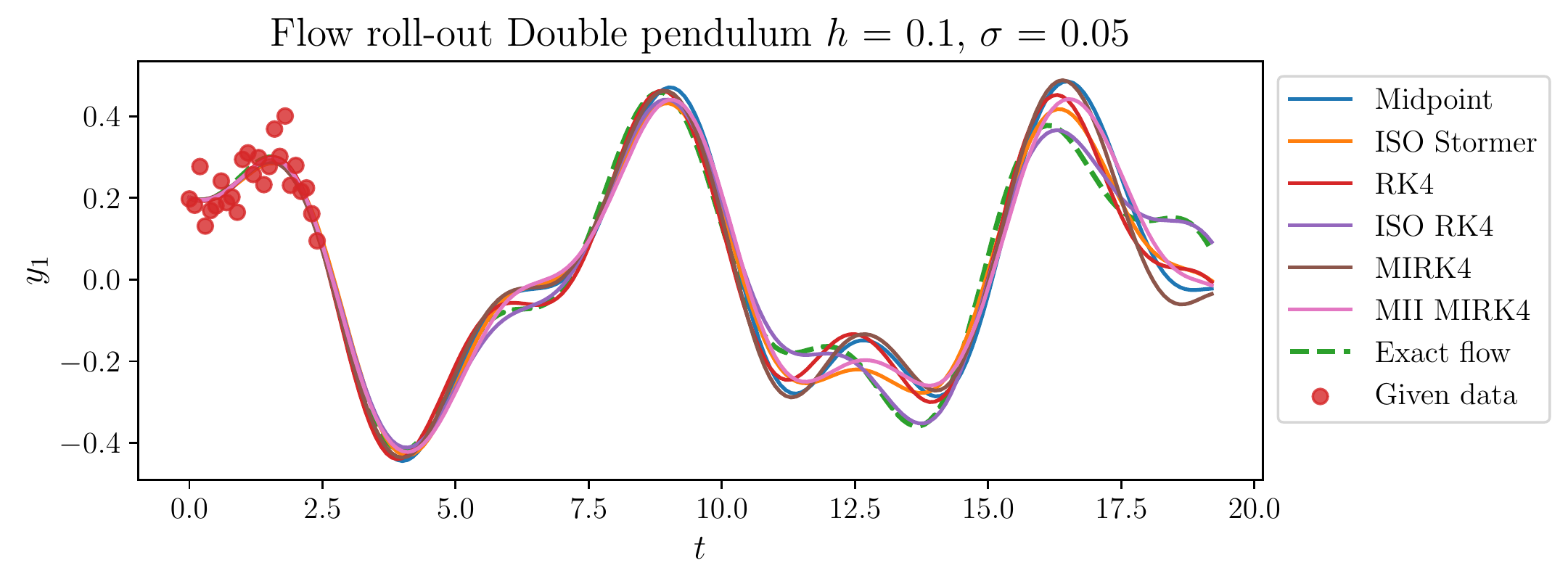}
  \includegraphics[trim=515 170 10 35,clip,width=0.1\textwidth]{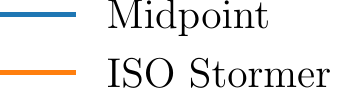}
  \includegraphics[trim=515 138 10 65,clip,width=0.1\textwidth]{plots/roll_out/roll-out_henon-heiles_j=1_sigma=0.05_dt=0.1.pdf}
  \includegraphics[trim=515 104 10 95,clip,width=0.1\textwidth]{plots/roll_out/roll-out_henon-heiles_j=1_sigma=0.05_dt=0.1.pdf}
    \includegraphics[trim=515 135 10 67,clip,width=0.1\textwidth]{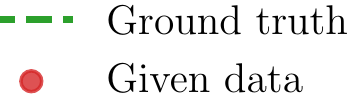}
  \caption{Roll-out in time obtained by integrating over the learned vector fields when training on data from the double pendulum Hamiltonian. }
    \label{fig:roll_out_time_dp}
      \vspace{-20pt}
\end{wrapfigure}

 \textbf{Methods and test problems: } 
 We train HNNs using different integrators and methods in the inverse problem \eqref{inverse_problem}. We use MIRK4 together with the MII method and compare to the implicit midpoint method, RK4 and MIRK4 applied as one-step methods, as well as ISO followed by Störmer--Verlet and RK4 integrated over multiple time-steps. The latter strategy, illustrated in Figure \ref{training_structure}, was suggested in \cite{Chen2020Symplectic}, where Störmer--Verlet is used. 
 Separable networks $H_{\theta}(q,p) = H_{1,\theta}(q) + H_{2,\theta}(p) $ are trained on data from the Fermi--Pasta--Ulam--Tsingou (FPUT) problem and the Hénon--Heiles system. For the double pendulum, which is non-separable, a fully connected network is used for all methods except Störmer--Verlet, which requires separability in order to be explicit. The Hamiltonians are described in Appendix \ref{test_problems} and all systems have solutions $y(t) \in \R^4$.
 
 After using the specified integrators in training, approximated solutions are computed for each learned vector field $f_{\theta}$ using the Scikit-learn implementation of DOP853 \cite{DORMAND198019}, which is also used to generate the training data. The error is averaged over $M = 10$ points and we find what we call the flow error by
  \begin{equation}
  \label{flow_error}
  \begin{split}
    e(f_{\theta}) &= \frac{1}{M} \sum_{n=1}^M  \|\hat y_n - y(t_n) \|_2, \quad y(t_n) \in S_M^{\text{test}},\\
    \hat y_{n+1} &= \Phi_{h,f_{\theta}}(y_n).
    \end{split}
  \end{equation}

 \textbf{Training data:} Training data is generated by sampling $N_2 = 300$ random initial values $y_0$ requiring that $0.3 \leq\| y_0\|_2\leq 0.6$. The data $S_{N_1,N_2}=\{ \tilde y_{n}^{(j)} \}_{n=0,j=0}^{N_1,N_2}$ is found by integrating the initial values with DOP853 with a tolerance of $10^{-15}$ for the following step sizes and number of steps: $(h,N_1) = (0.4,4),(0.2,8),(0.1,16)$. The points in the flow are perturbed by noise where $\sigma \in\{0,0.05\}$. Error is measured in $M=10$ random points in the flow, within the same domain as the initial values. Furthermore, experiments are repeated with a new random seed for the generation of data and initialization of neural network parameters five times in order to compute the standard deviation of the flow error. The flow error is shown in Figure \ref{fig:flow_error_noise}. Additional results are presented in Appendix \ref{appendix:more_num_results}.

 \textbf{Neural network architecture and optimization:}  For all test problems, the neural networks have $3$ layers with a width of $200$ neurons and $\text{tanh}( \cdot )$ as the activation function. The algorithms are implemented using PyTorch \cite{pytorch} and the code for performing ISO is a modification of the implementation by \cite{Chen2020Symplectic}\footnote{h\url{ttps://github.com/zhengdao-chen/SRNN} (CC-BY-NC 4.0 License)}. Training is done using the quasi-Newton L-BFGS algorithm \cite{nocedal1999numerical} for $20$ epochs without batching. Further details are provided in Appendix \ref{appendix:details_nn_train} and the code could be found at \href{https://github.com/hakonnoren/learning_hamiltonian_noise/}{github.com/hakonnoren/learning{\_}hamiltonian{\_}noise}.

  \begin{figure}[htb]
  \centering
    \includegraphics[trim=0 0 0 0,clip,width=1\textwidth]{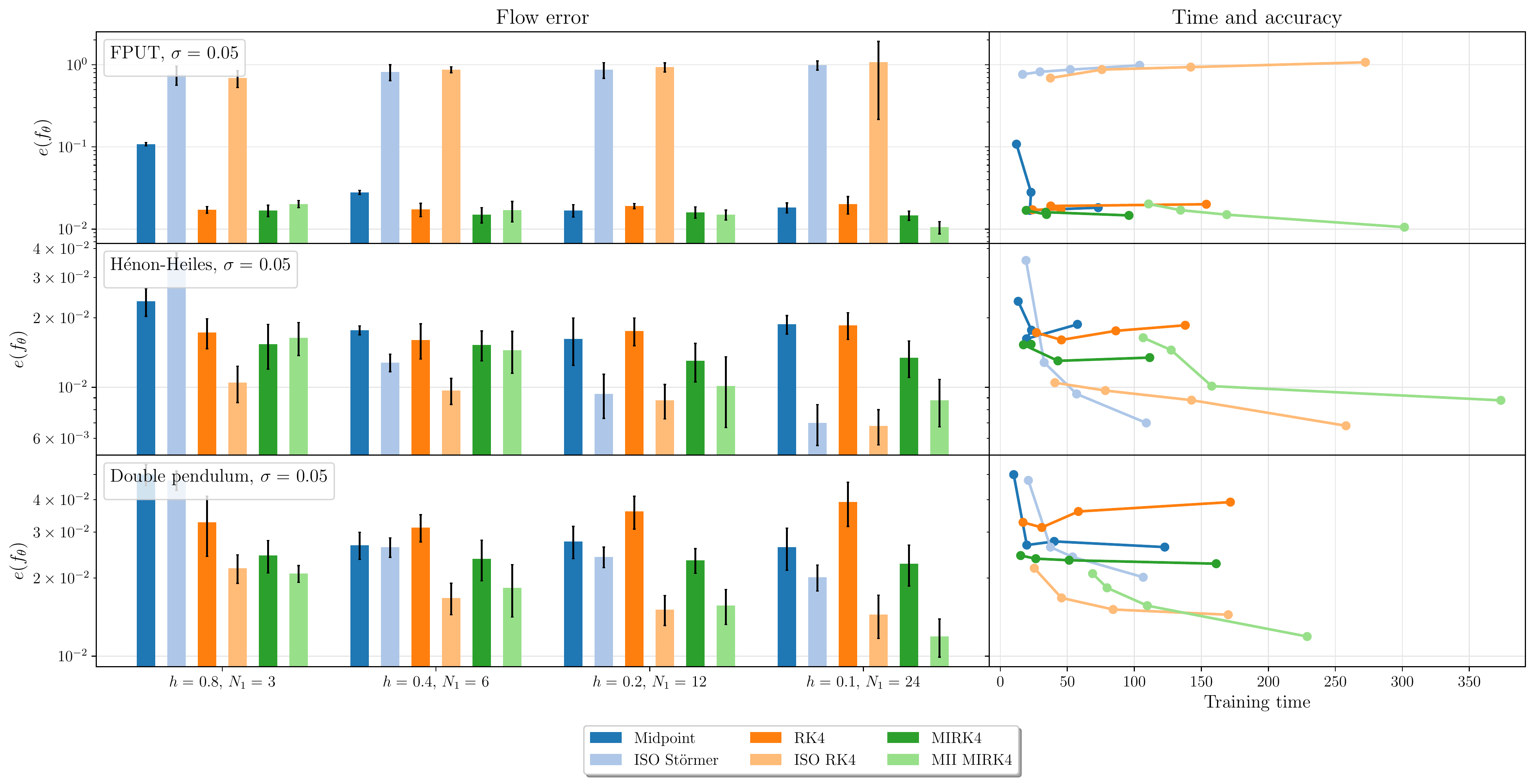}
  \caption{The flow error when learning vector fields using one-step methods directly (Midpoint, RK4 and MIRK4), ISO and multiple time-steps (ISO Störmer and ISO RK4) and MII (MII MIRK4). The error bars display the standard deviation after rerunning 5 experiments on data with $\sigma = 0.05$. The right subplot shows the computational time used in training against the flow error. }

    \label{fig:flow_error_noise}
\end{figure}

\textbf{Results:} As observed in Figure \ref{fig:flow_error_noise} and supported by the analytical result illustrated in Figure \ref{fig:propagation of noise}, the MII approach facilitates more accurate training from noisy data than one-step methods. However, training with multiple integration steps in combination with ISO yields lower errors when RK4 is used for the Hénon--Heiles problem and similar performance as MII on the double pendulum. We notice that the SRNN approach, i.e.\ ISO with Störmer--Verlet, is improved when switching to RK4, which means sacrificing symplecticity to achieve higher order. The results for FPUT stand out in Figure \ref{fig:flow_error_noise}, since both ISO methods have large errors here. The roll-out in time of the learned vector fields is presented in Figure \ref{fig:roll_out_time} in Appendix \ref{appendix:more_num_results}, where the same can be observed. As also could be seen here, the FPUT Hamiltonian gives rise to highly oscillatory trajectories, and the errors observed in Figure \ref{fig:flow_error_noise} might indicate that ISO is ill-suited for this kind of dynamical systems. 

Two observations could be made regarding the one-step methods without averaging or ISO. First, it is likely that the midpoint method has weaker performance for large step sizes due to its lower order, compared to both RK4 and MIRK4, despite the fact that it is a symplectic method. The same is clear from Figure \ref{fig:flow_error_no_noise} in Appendix \ref{appendix:more_num_results}, which display the flow error when training on data without noise. Secondly, building on the sensitivity analysis, we observe that MIRK4 consistently attains higher accuracy than RK4, as expected from the Monte-Carlo simulation found in Figure \ref{fig:propagation of noise}.

\section{Conclusion}

In this work we present the mean inverse integrator, which allows both chaotic and oscillatory dynamical systems to be learned with high accuracy from noisy data. Within this method, integrators of the MIRK class are a key component. To analyse how noise is propagated when training with MII and MIRK, compared to much used explicit methods such as RK4, we developed a sensitivity analysis that is verified both by a Monte-Carlo approximation and reflected in the error of the learned vector fields. Finally, we build on the SRNN \cite{Chen2020Symplectic} by replacing Störmer--Verlet with RK4, and observer increased performance. When also considering the weak performance of the implicit midpoint method, this tells us that order might be of greater importance than preserving the symplectic structure when training HNNs. Both the MIRK methods, the mean inverse integrator and initial state optimization form building blocks that could be combined to form novel approaches for solving inverse problems and learning from noisy data. 

\textbf{Limitations:} The experiments presented here assume that both the generalized coordinates $q_n$ and the generalized momenta $p_n$ could be observed. In a setting where HNNs are to model real and not simulated data, the observations might lack generalized momenta \cite{chen2021neural_symp} or follow Cartesian coordinates, requiring the enforcement of constraints \cite{celledoni23lho,SimplifyingHamiltonianandLagrangianNeuralNetworksviaExplicitConstraints}. Combining approaches that are suitable for data that is both noisy and follow less trivial coordinate systems is a subject for future research.

\newpage

\bibliographystyle{unsrtnat}
\bibliography{ref}

\begin{thebibliography}{49}
\providecommand{\natexlab}[1]{#1}
\providecommand{\url}[1]{\texttt{#1}}
\expandafter\ifx\csname urlstyle\endcsname\relax
  \providecommand{\doi}[1]{doi: #1}\else
  \providecommand{\doi}{doi: \begingroup \urlstyle{rm}\Url}\fi

\bibitem[Raissi et~al.(2019)Raissi, Perdikaris, and
  Karniadakis]{raissi2019physics}
Maziar Raissi, Paris Perdikaris, and George~E Karniadakis.
\newblock Physics-informed neural networks: A deep learning framework for
  solving forward and inverse problems involving nonlinear partial differential
  equations.
\newblock \emph{Journal of Computational physics}, 378:\penalty0 686--707,
  2019.

\bibitem[Rackauckas et~al.(2020)Rackauckas, Ma, Martensen, Warner, Zubov,
  Supekar, Skinner, Ramadhan, and Edelman]{rackauckas2020universal}
Christopher Rackauckas, Yingbo Ma, Julius Martensen, Collin Warner, Kirill
  Zubov, Rohit Supekar, Dominic Skinner, Ali Ramadhan, and Alan Edelman.
\newblock Universal differential equations for scientific machine learning.
\newblock \emph{arXiv preprint arXiv:2001.04385}, 2020.

\bibitem[Chen et~al.(2018)Chen, Rubanova, Bettencourt, and
  Duvenaud]{chen2018neural}
Ricky~TQ Chen, Yulia Rubanova, Jesse Bettencourt, and David~K Duvenaud.
\newblock Neural ordinary differential equations.
\newblock \emph{Advances in neural information processing systems}, 31, 2018.

\bibitem[Li et~al.(2020)Li, Kovachki, Azizzadenesheli, Liu, Bhattacharya,
  Stuart, and Anandkumar]{li2020fourier}
Zongyi Li, Nikola Kovachki, Kamyar Azizzadenesheli, Burigede Liu, Kaushik
  Bhattacharya, Andrew Stuart, and Anima Anandkumar.
\newblock Fourier neural operator for parametric partial differential
  equations.
\newblock \emph{arXiv preprint arXiv:2010.08895}, 2020.

\bibitem[Greydanus et~al.(2019)Greydanus, Dzamba, and
  Yosinski]{HamiltonianNeuralNetworks}
Sam Greydanus, Misko Dzamba, and Jason Yosinski.
\newblock {H}amiltonian neural networks.
\newblock \emph{CoRR}, abs/1906.01563, 2019.
\newblock URL \url{http://arxiv.org/abs/1906.01563}.

\bibitem[Goldstein et~al.(2001)Goldstein, Poole, and
  Safko]{GoldsteinHerbert2014Cm}
Herbert Goldstein, Charles Poole, and John Safko.
\newblock \emph{Classical {M}echanics}.
\newblock Addison Wesley, 3 edition, 2001.

\bibitem[Leimkuhler and Reich(2005)]{leimkuhler_reich_2005}
Benedict Leimkuhler and Sebastian Reich.
\newblock \emph{Simulating {H}amiltonian Dynamics}.
\newblock Cambridge Monographs on Applied and Computational Mathematics.
  Cambridge University Press, 2005.
\newblock \doi{10.1017/CBO9780511614118}.

\bibitem[Hairer et~al.(2006)Hairer, Lubich, and Wanner]{hairer2006geometric}
Ernst Hairer, Christian Lubich, and Gerhard Wanner.
\newblock \emph{{Geometric {N}umerical {I}ntegration: {S}tructure-{P}reserving
  {A}lgorithms for {O}rdinary {D}ifferential {E}quations; 2nd ed.}}
\newblock Springer, Dordrecht, 2006.
\newblock \doi{10.1007/3-540-30666-8}.

\bibitem[Offen and Ober-Bl{\"o}baum(2022)]{offen2022symplectic}
Christian Offen and Sina Ober-Bl{\"o}baum.
\newblock Symplectic integration of learned {H}amiltonian systems.
\newblock \emph{Chaos: An Interdisciplinary Journal of Nonlinear Science},
  32\penalty0 (1):\penalty0 013122, 2022.

\bibitem[Chen et~al.(2020)Chen, Zhang, Arjovsky, and
  Bottou]{Chen2020Symplectic}
Zhengdao Chen, Jianyu Zhang, Martin Arjovsky, and Léon Bottou.
\newblock Symplectic recurrent neural networks.
\newblock In \emph{International Conference on Learning Representations}, 2020.
\newblock URL \url{https://openreview.net/forum?id=BkgYPREtPr}.

\bibitem[Zhu et~al.(2020{\natexlab{a}})Zhu, Jin, and
  Tang]{DeepHamiltoniannetworksbasedonsymplecticintegrators}
Aiqing Zhu, Pengzhan Jin, and Yifa Tang.
\newblock Deep {H}amiltonian networks based on symplectic integrators.
\newblock \emph{arXiv preprint arXiv:2004.13830}, 2020{\natexlab{a}}.

\bibitem[David and
  M{\'e}hats(2021)]{SymplecticLearningforHamiltonianNeuralNetworks}
Marco David and Florian M{\'e}hats.
\newblock Symplectic learning for {H}amiltonian neural networks.
\newblock \emph{arXiv preprint arXiv:2106.11753}, 2021.

\bibitem[Jin et~al.(2020)Jin, Zhang, Zhu, Tang, and
  Karniadakis]{SympNetsIntrinsicstructure-preservingsymplecticnetworksforidentifyingHamiltoniansystems}
Pengzhan Jin, Zhen Zhang, Aiqing Zhu, Yifa Tang, and George~Em Karniadakis.
\newblock Symp{N}ets: {I}ntrinsic structure-preserving symplectic networks for
  identifying {H}amiltonian systems.
\newblock \emph{Neural Networks}, 132:\penalty0 166--179, 2020.

\bibitem[Zhu et~al.(2020{\natexlab{b}})Zhu, Jin, Zhu, and
  Tang]{Inversemodifieddifferentialequationsfordiscoveryofdynamics}
Aiqing Zhu, Pengzhan Jin, Beibei Zhu, and Yifa Tang.
\newblock Inverse modified differential equations for discovery of dynamics.
\newblock \emph{arXiv preprint arXiv:2009.01058}, 2020{\natexlab{b}}.

\bibitem[Matsubara et~al.(2020)Matsubara, Ishikawa, and
  Yaguchi]{DeepEnergy-BasedModelingofDiscrete-TimePhysics}
Takashi Matsubara, Ai~Ishikawa, and Takaharu Yaguchi.
\newblock Deep energy-based modeling of discrete-time physics.
\newblock \emph{Advances in Neural Information Processing Systems},
  33:\penalty0 13100--13111, 2020.

\bibitem[Eidnes(2022)]{eidnes2022order}
S{\o}lve Eidnes.
\newblock Order theory for discrete gradient methods.
\newblock \emph{BIT}, 62\penalty0 (4):\penalty0 1207--1255, 2022.
\newblock ISSN 0006-3835.
\newblock \doi{10.1007/s10543-022-00909-z}.
\newblock URL \url{https://doi.org/10.1007/s10543-022-00909-z}.

\bibitem[Celledoni et~al.(2023)Celledoni, Leone, Murari, and
  Owren]{celledoni23lho}
Elena Celledoni, Andrea Leone, Davide Murari, and Brynjulf Owren.
\newblock Learning {H}amiltonians of constrained mechanical systems.
\newblock \emph{J. Comput. Appl. Math.}, 417:\penalty0 Paper No. 114608, 12,
  2023.
\newblock ISSN 0377-0427.
\newblock \doi{10.1016/j.cam.2022.114608}.
\newblock URL \url{https://doi.org/10.1016/j.cam.2022.114608}.

\bibitem[Sanchez-Gonzalez et~al.(2019)Sanchez-Gonzalez, Bapst, Cranmer, and
  Battaglia]{sanchez-gonzalezHamiltonianGraphNetworks2019}
Alvaro Sanchez-Gonzalez, Victor Bapst, Kyle Cranmer, and Peter Battaglia.
\newblock Hamiltonian graph networks with ode integrators.
\newblock \emph{arXiv preprint arXiv:1909.12790}, 2019.

\bibitem[Wanner and Hairer(1996)]{wanner1996solving}
Gerhard Wanner and Ernst Hairer.
\newblock \emph{Solving ordinary differential equations II}, volume 375.
\newblock Springer Berlin Heidelberg, 1996.

\bibitem[Liang et~al.(2022)Liang, Huang, and Zhang]{liang_stiffness-aware_2022}
Senwei Liang, Zhongzhan Huang, and Hong Zhang.
\newblock Stiffness-aware neural network for learning hamiltonian systems.
\newblock In \emph{International Conference on Learning Representations}, 2022.

\bibitem[Cash(1975)]{cash1975class}
Jeff~R Cash.
\newblock A class of implicit {Runge--Kutta} methods for the numerical
  integration of stiff ordinary differential equations.
\newblock \emph{Journal of the ACM (JACM)}, 22\penalty0 (4):\penalty0 504--511,
  1975.

\bibitem[Burrage et~al.(1994)Burrage, Chipman, and Muir]{burrage1994order}
K~Burrage, FH~Chipman, and Paul~H Muir.
\newblock Order results for mono-implicit {R}unge--{K}utta methods.
\newblock \emph{SIAM journal on numerical analysis}, 31\penalty0 (3):\penalty0
  876--891, 1994.

\bibitem[Noren(2023)]{noren2023learning}
H{\aa}kon Noren.
\newblock Learning {H}amiltonian systems with mono-implicit {R}unge--{K}utta
  methods.
\newblock \emph{arXiv preprint, arXiv:2303.03769}, 2023.

\bibitem[Sharma et~al.(2022)Sharma, Galioto, Gorodetsky, and
  Kramer]{sharmaBayesianIdentificationNonseparable2022}
Harsh Sharma, Nicholas Galioto, Alex~A Gorodetsky, and Boris Kramer.
\newblock Bayesian identification of nonseparable {H}amiltonian systems using
  stochastic dynamic models.
\newblock In \emph{2022 IEEE 61st Conference on Decision and Control (CDC)},
  pages 6742--6749. IEEE, 2022.

\bibitem[van Bokhoven(1980)]{vanBokhoven1980efficient}
W.~M.~G. van Bokhoven.
\newblock Efficient higher order implicit one-step methods for integration of
  stiff differential equations.
\newblock \emph{BIT}, 20\penalty0 (1):\penalty0 34--43, 1980.
\newblock ISSN 0006-3835.
\newblock \doi{10.1007/BF01933583}.
\newblock URL \url{https://doi.org/10.1007/BF01933583}.

\bibitem[Cash and Singhal(1982{\natexlab{a}})]{Cash1982mono}
J.~R. Cash and A.~Singhal.
\newblock Mono-implicit {R}unge--{K}utta formulae for the numerical integration
  of stiff differential systems.
\newblock \emph{IMA J. Numer. Anal.}, 2\penalty0 (2):\penalty0 211--227,
  1982{\natexlab{a}}.
\newblock ISSN 0272-4979.
\newblock \doi{10.1093/imanum/2.2.211}.

\bibitem[Eidnes et~al.(2022)Eidnes, Stasik, Sterud, B{\o}hn, and
  Riemer-S{\o}rensen]{PortHam_eidnes}
S{\o}lve Eidnes, Alexander~J Stasik, Camilla Sterud, Eivind B{\o}hn, and Signe
  Riemer-S{\o}rensen.
\newblock Pseudo-{H}amiltonian neural networks with state-dependent external
  forces.
\newblock \emph{arXiv preprint, arXiv:2206.02660}, 2022.

\bibitem[Yoshida(1990)]{yoshida1990construction}
Haruo Yoshida.
\newblock Construction of higher order symplectic integrators.
\newblock \emph{Physics letters A}, 150\penalty0 (5-7):\penalty0 262--268,
  1990.

\bibitem[Desai et~al.(2021)Desai, Mattheakis, and
  Roberts]{desai2021variational}
Shaan~A Desai, Marios Mattheakis, and Stephen~J Roberts.
\newblock Variational integrator graph networks for learning energy-conserving
  dynamical systems.
\newblock \emph{Physical Review E}, 104\penalty0 (3):\penalty0 035310, 2021.

\bibitem[DiPietro et~al.(2020)DiPietro, Xiong, and Zhu]{dipietro2020sparse}
Daniel DiPietro, Shiying Xiong, and Bo~Zhu.
\newblock Sparse symplectically integrated neural networks.
\newblock \emph{Advances in Neural Information Processing Systems},
  33:\penalty0 6074--6085, 2020.

\bibitem[Quispel and Turner(1996)]{quispel1996discrete}
GRW Quispel and Grant~S Turner.
\newblock Discrete gradient methods for solving {ODEs} numerically while
  preserving a first integral.
\newblock \emph{Journal of Physics A: Mathematical and General}, 29\penalty0
  (13):\penalty0 L341, 1996.

\bibitem[McLachlan et~al.(1999)McLachlan, Quispel, and
  Robidoux]{mclachlan1999geometric}
Robert~I McLachlan, G~Reinout~W Quispel, and Nicolas Robidoux.
\newblock Geometric integration using discrete gradients.
\newblock \emph{Philosophical Transactions of the Royal Society of London.
  Series A: Mathematical, Physical and Engineering Sciences}, 357\penalty0
  (1754):\penalty0 1021--1045, 1999.

\bibitem[Murua(1995)]{uria1995metodos}
Ander Murua.
\newblock \emph{M{\'e}todos simpl{\'e}cticos desarrollables en P-series}.
\newblock PhD thesis, PhD thesis. Valladolid: Universidad de Valladolid, 1995.

\bibitem[Chartier et~al.(2007{\natexlab{a}})Chartier, Hairer, and
  Vilmart]{chartier_numerical_2007}
Philippe Chartier, Ernst Hairer, and Gilles Vilmart.
\newblock Numerical integrators based on modified differential equations.
\newblock \emph{Mathematics of Computation}, 76\penalty0 (260):\penalty0
  1941--1953, October 2007{\natexlab{a}}.
\newblock ISSN 0025-5718, 1088-6842.
\newblock \doi{10.1090/S0025-5718-07-01967-9}.
\newblock URL
  \url{https://www.ams.org/mcom/2007-76-260/S0025-5718-07-01967-9/}.

\bibitem[Dormand and Prince(1980)]{DORMAND198019}
J.R. Dormand and P.J. Prince.
\newblock A family of embedded {R}unge--{K}utta formulae.
\newblock \emph{Journal of Computational and Applied Mathematics}, 6\penalty0
  (1):\penalty0 19--26, 1980.
\newblock ISSN 0377-0427.
\newblock \doi{https://doi.org/10.1016/0771-050X(80)90013-3}.
\newblock URL
  \url{https://www.sciencedirect.com/science/article/pii/0771050X80900133}.

\bibitem[Paszke et~al.(2019)Paszke, Gross, Massa, Lerer, Bradbury, Chanan,
  Killeen, Lin, Gimelshein, Antiga, et~al.]{pytorch}
Adam Paszke, Sam Gross, Francisco Massa, Adam Lerer, James Bradbury, Gregory
  Chanan, Trevor Killeen, Zeming Lin, Natalia Gimelshein, Luca Antiga, et~al.
\newblock Py{T}orch: {A}n imperative style, high-performance deep learning
  library.
\newblock \emph{Advances in neural information processing systems},
  32:\penalty0 8026--8037, 2019.

\bibitem[Nocedal and Wright(1999)]{nocedal1999numerical}
Jorge Nocedal and Stephen~J Wright.
\newblock \emph{Numerical optimization}.
\newblock Springer, 1999.

\bibitem[Chen et~al.(2021)Chen, Matsubara, and Yaguchi]{chen2021neural_symp}
Yuhan Chen, Takashi Matsubara, and Takaharu Yaguchi.
\newblock Neural symplectic form: learning hamiltonian equations on general
  coordinate systems.
\newblock \emph{Advances in Neural Information Processing Systems},
  34:\penalty0 16659--16670, 2021.

\bibitem[Finzi et~al.(2020)Finzi, Wang, and
  Wilson]{SimplifyingHamiltonianandLagrangianNeuralNetworksviaExplicitConstraints}
Marc Finzi, Ke~Alexander Wang, and Andrew~Gordon Wilson.
\newblock Simplifying {H}amiltonian and {L}agrangian neural networks via
  explicit constraints.
\newblock \emph{arXiv preprint arXiv:2010.13581}, 2020.

\bibitem[Fermi et~al.(1955)Fermi, Pasta, Ulam, and Tsingou]{fermi1955studies}
Enrico Fermi, P~Pasta, Stanislaw Ulam, and Mary Tsingou.
\newblock Studies of the nonlinear problems.
\newblock Technical report, Los Alamos National Lab.(LANL), Los Alamos, NM
  (United States), 1955.

\bibitem[Hairer et~al.(1993)Hairer, Nørsett, and Wanner]{hairer93sodI}
E.~Hairer, S.~P. Nørsett, and G.~Wanner.
\newblock \emph{Solving ordinary differential equations. {I}}, volume~8 of
  \emph{Springer Series in Computational Mathematics}.
\newblock Springer-Verlag, Berlin, second edition, 1993.
\newblock ISBN 3-540-56670-8.
\newblock Nonstiff problems.

\bibitem[Muir(1999)]{muir_optimal_nodate}
P.~H. Muir.
\newblock Optimal discrete and continuous mono-implicit {R}unge-{K}utta schemes
  for {BVODE}s.
\newblock \emph{Adv. Comput. Math.}, 10\penalty0 (2):\penalty0 135--167, 1999.
\newblock ISSN 1019-7168.
\newblock \doi{10.1023/A:1018926631734}.
\newblock URL \url{https://doi.org/10.1023/A:1018926631734}.

\bibitem[Cash and Moore(1980)]{Cash1980high}
J.~R. Cash and D.~R. Moore.
\newblock A high order method for the numerical solution of two-point boundary
  value problems.
\newblock \emph{BIT}, 20\penalty0 (1):\penalty0 44--52, 1980.
\newblock ISSN 0006-3835.
\newblock \doi{10.1007/BF01933584}.
\newblock URL \url{https://doi.org/10.1007/BF01933584}.

\bibitem[Chartier(2015)]{chartier_symmetric_2015}
Philippe Chartier.
\newblock Symmetric {Methods}.
\newblock In Björn Engquist, editor, \emph{Encyclopedia of {Applied} and
  {Computational} {Mathematics}}, pages 1439--1448. Springer, Berlin,
  Heidelberg, 2015.
\newblock ISBN 978-3-540-70529-1.
\newblock \doi{10.1007/978-3-540-70529-1_151}.
\newblock URL \url{https://doi.org/10.1007/978-3-540-70529-1_151}.

\bibitem[Cash and Singhal(1982{\natexlab{b}})]{Cash1982high}
J.~R. Cash and A.~Singhal.
\newblock High order methods for the numerical solution of two-point boundary
  value problems.
\newblock \emph{BIT}, 22\penalty0 (2):\penalty0 184--199, 1982{\natexlab{b}}.
\newblock ISSN 0006-3835.
\newblock \doi{10.1007/BF01944476}.
\newblock URL \url{https://doi.org/10.1007/BF01944476}.

\bibitem[Virtanen et~al.(2020)Virtanen, Gommers, Oliphant, Haberland, Reddy,
  Cournapeau, Burovski, Peterson, Weckesser, Bright, {van der Walt}, Brett,
  Wilson, Millman, Mayorov, Nelson, Jones, Kern, Larson, Carey, Polat, Feng,
  Moore, {VanderPlas}, Laxalde, Perktold, Cimrman, Henriksen, Quintero, Harris,
  Archibald, Ribeiro, Pedregosa, {van Mulbregt}, and {SciPy 1.0
  Contributors}]{2020SciPyNMeth}
Pauli Virtanen, Ralf Gommers, Travis~E. Oliphant, Matt Haberland, Tyler Reddy,
  David Cournapeau, Evgeni Burovski, Pearu Peterson, Warren Weckesser, Jonathan
  Bright, St{\'e}fan~J. {van der Walt}, Matthew Brett, Joshua Wilson, K.~Jarrod
  Millman, Nikolay Mayorov, Andrew R.~J. Nelson, Eric Jones, Robert Kern, Eric
  Larson, C~J Carey, {\.I}lhan Polat, Yu~Feng, Eric~W. Moore, Jake
  {VanderPlas}, Denis Laxalde, Josef Perktold, Robert Cimrman, Ian Henriksen,
  E.~A. Quintero, Charles~R. Harris, Anne~M. Archibald, Ant{\^o}nio~H. Ribeiro,
  Fabian Pedregosa, Paul {van Mulbregt}, and {SciPy 1.0 Contributors}.
\newblock {{SciPy} 1.0: Fundamental Algorithms for Scientific Computing in
  Python}.
\newblock \emph{Nature Methods}, 17:\penalty0 261--272, 2020.
\newblock \doi{10.1038/s41592-019-0686-2}.

\bibitem[Chartier et~al.(2007{\natexlab{b}})Chartier, Hairer, and
  Vilmart]{Chartier07}
Philippe Chartier, Ernst Hairer, and Gilles Vilmart.
\newblock Numerical integrators based on modified differential equations.
\newblock \emph{Math. Comp.}, 76\penalty0 (260):\penalty0 1941--1953,
  2007{\natexlab{b}}.
\newblock ISSN 0025-5718.
\newblock \doi{10.1090/S0025-5718-07-01967-9}.
\newblock URL \url{https://doi.org/10.1090/S0025-5718-07-01967-9}.

\bibitem[Zhong and Marsden(1988)]{Zhong1988Lie}
Ge~Zhong and Jerrold~E. Marsden.
\newblock Lie-{P}oisson {H}amilton-{J}acobi theory and {L}ie-{P}oisson
  integrators.
\newblock \emph{Phys. Lett. A}, 133\penalty0 (3):\penalty0 134--139, 1988.
\newblock ISSN 0375-9601.
\newblock \doi{10.1016/0375-9601(88)90773-6}.
\newblock URL \url{https://doi.org/10.1016/0375-9601(88)90773-6}.

\bibitem[Matsubara and Yaguchi(2022)]{Matsubara2022finde}
Takashi Matsubara and Takaharu Yaguchi.
\newblock {FINDE}: Neural differential equations for finding and preserving
  invariant quantities.
\newblock \emph{arXiv preprint, arXiv:2210.00272}, 2022.

\end{thebibliography}

\appendix

\section{Test problems}\label{test_problems}

\textbf{Fermi--Pasta--Ulam--Tsingou:} This dynamical system is a model for a chain of  $2m +1 $ alternating stiff and soft springs connecting $2m$ mass points. The chain is fixed in both ends \cite{hairer2006geometric,fermi1955studies}. With the coordinate transformation suggested in \cite[Ch. I.5.I]{hairer2006geometric} we have coordinates $[q,p]^T \in \R^{4m}$ where $q_i, i = 1,\dots,m$ represents a scaled displacement of the $i$-th stiff spring and $q_{i+m}, i = 1,\dots,m$ represents a scaled expansion of the $i$-th spring. $q_i$ represents their velocities. Letting $\omega$ be the angular frequency of the stiff spring, in general the Hamiltonian is given by 
\begin{align*}
    H(q,p) =& \;  \frac{1}{2}\sum_{i=1}^m\big( p_i^2 + p_{i+m}^2 \big) + \frac{\omega^2}{2}\sum_{i=1}^mq_{i+m}^2 \\
    &+ \frac{1}{4}\bigg( \sum_{i=1}^{m-1}\big (q_{i+1} - q_{i+m+1} - q_{i}-q_{i+m}\big )^4 + (q_{1} - q_{m+1})^4 + (q_{m} + q_{2m})^4\bigg)
\end{align*}
We consider the most trivial case of $m=1$ and letting $\omega=2$, yielding the quartic, separable Hamiltonian by
\begin{align*}
    H(q_1,q_2,p_1,p_2) =& \frac{1}{2}\big( p_1^2 + p_2^2 \big) + 2q_2^2
    + \frac{1}{4}\bigg( \big(q_{1} - q_{2}\big)^4 + \big(q_{1} + q_{2}\big)^4 \bigg).
\end{align*}

\textbf{Double pendulum:} Let $q_i$ and $p_i$ denote the angle and angular momentum of pendulum $i = 1,2$. The double pendulum system has a Hamiltonian that is not separable, where  $y=[q_1,q_2,p_1,p_2]^T \in \R^4$ and the Hamiltonian is given by
\begin{equation*}
H(q_1,q_2,p_1,p_2) = \frac{\frac{1}{2}p_1^2 + p_2^2 - p_1p_2\cos(q_1-q_2)}{1+\sin^2(q_1-q_2)} - 2\cos(q_1) -\cos(q_2).
\end{equation*}

\textbf{Hénon--Heiles:} This model was introduced for describing stellar motion inside the gravitational potential of a galaxy, as described in \cite{hairer2006geometric}. This Hamiltonian is separable. However, it is a canonical example of a chaotic system and its properties are discussed more in detail in \cite{GoldsteinHerbert2014Cm}. The Hamiltonian is given by
  \begin{equation*}
    H(q_1,q_2,p_1,p_2) = \frac{1}{2}(p_1^2 + p_2^2) + \frac{1}{2}(q_1^2 + q_2^2) + q_1^2q_2 - \frac{1}{3}q_2^3.
  \end{equation*}

\section{Additional numerical results}\label{appendix:more_num_results}

Here we present additional numerical experiments. In Figure \ref{fig:flow_error_no_noise}, the flow error when learning from data without noise, could be found. The roll-out in time of the learned Hamiltonian for the FPUT and Hénon--Heiles problem is presented in Figure \ref{fig:roll_out_time}.

\begin{figure}[htb]
  \centering
    \includegraphics[trim=0 0 0 0,clip,width=1\textwidth]{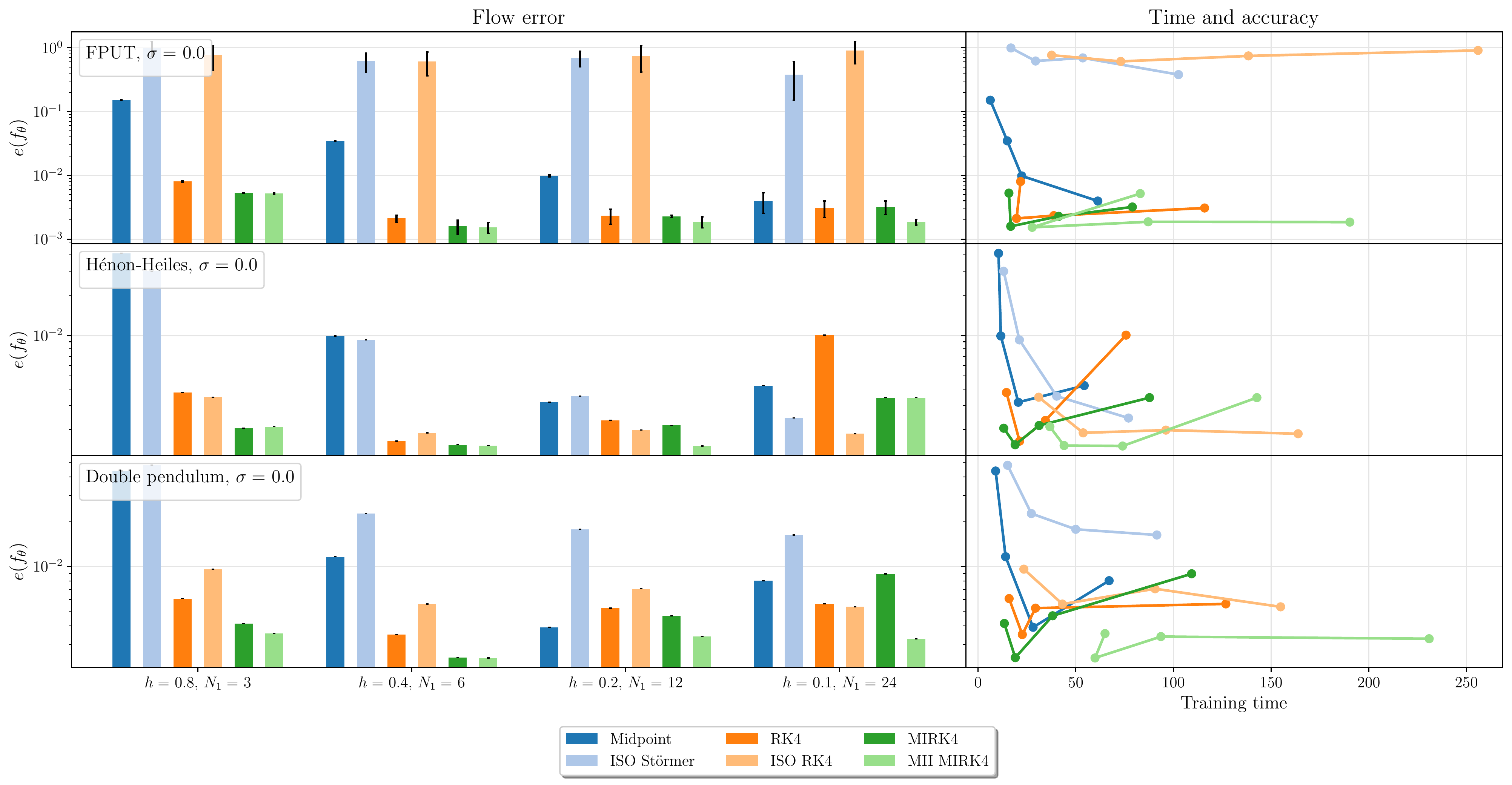}
  \caption{The flow error when learning vector fields using one-step methods directly (Midpoint, RK4 and MIRK4), ISO and multiple time-steps (ISO Störmer and ISO RK4) and MII (MII MIRK4). The error bars display the standard deviation after rerunning 5 experiments on data with $\sigma = 0$. The right subplot shows the computational time used in training against the flow error.}
    \label{fig:flow_error_no_noise}
\end{figure}

\begin{figure}[htb]
  \centering
    \includegraphics[trim=0 0 116 0,clip,width=0.45\textwidth]{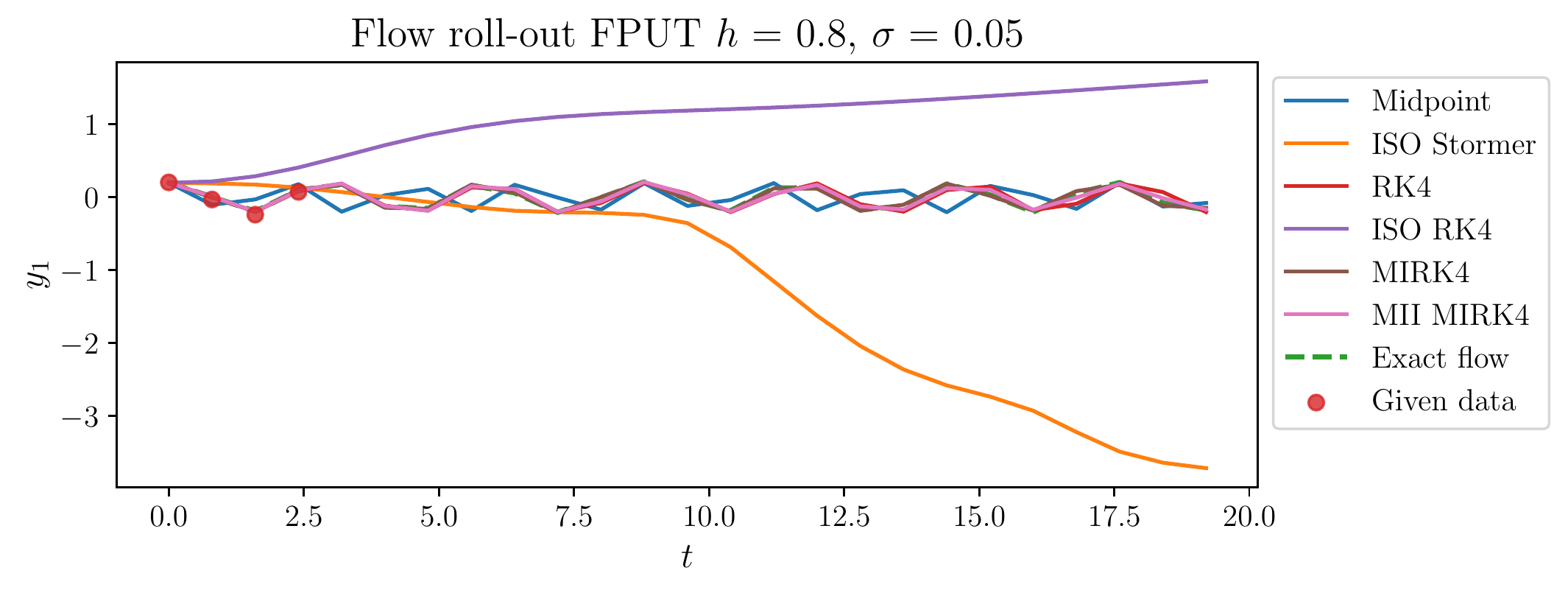}
    \includegraphics[trim=0 0 116 0,clip,width=0.45\textwidth]{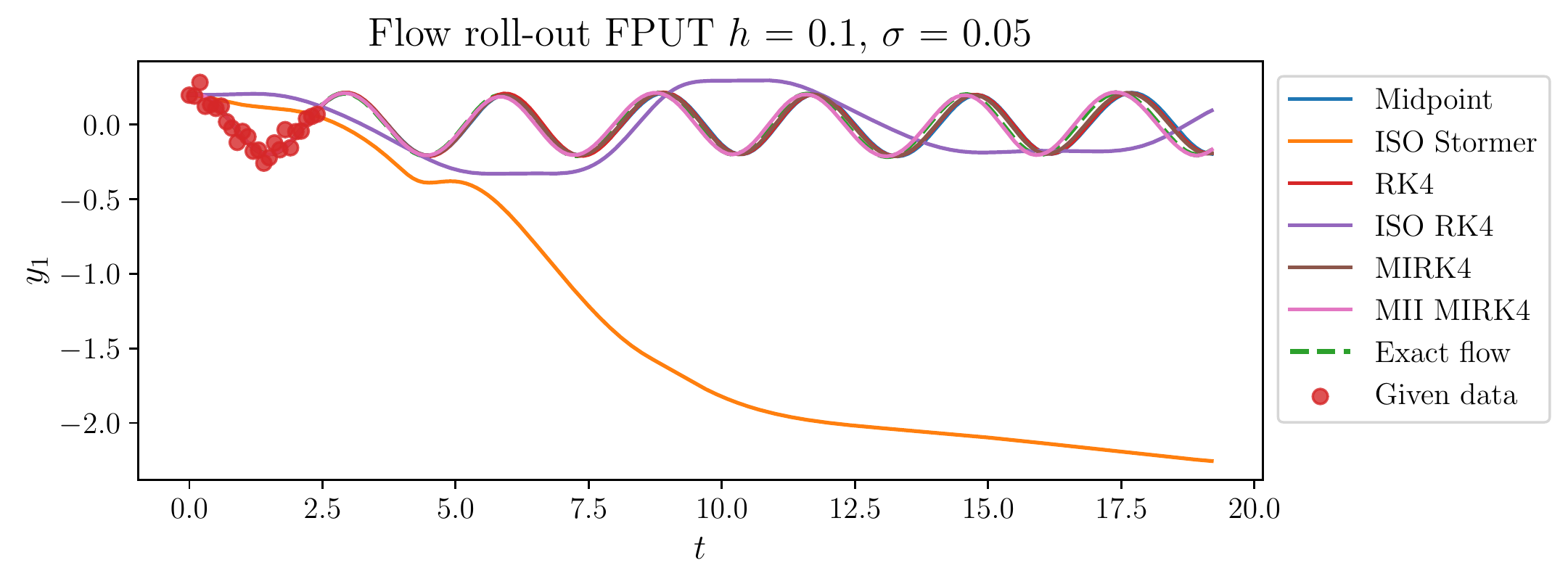}
    \includegraphics[trim=0 0 116 0,clip,width=0.45\textwidth]{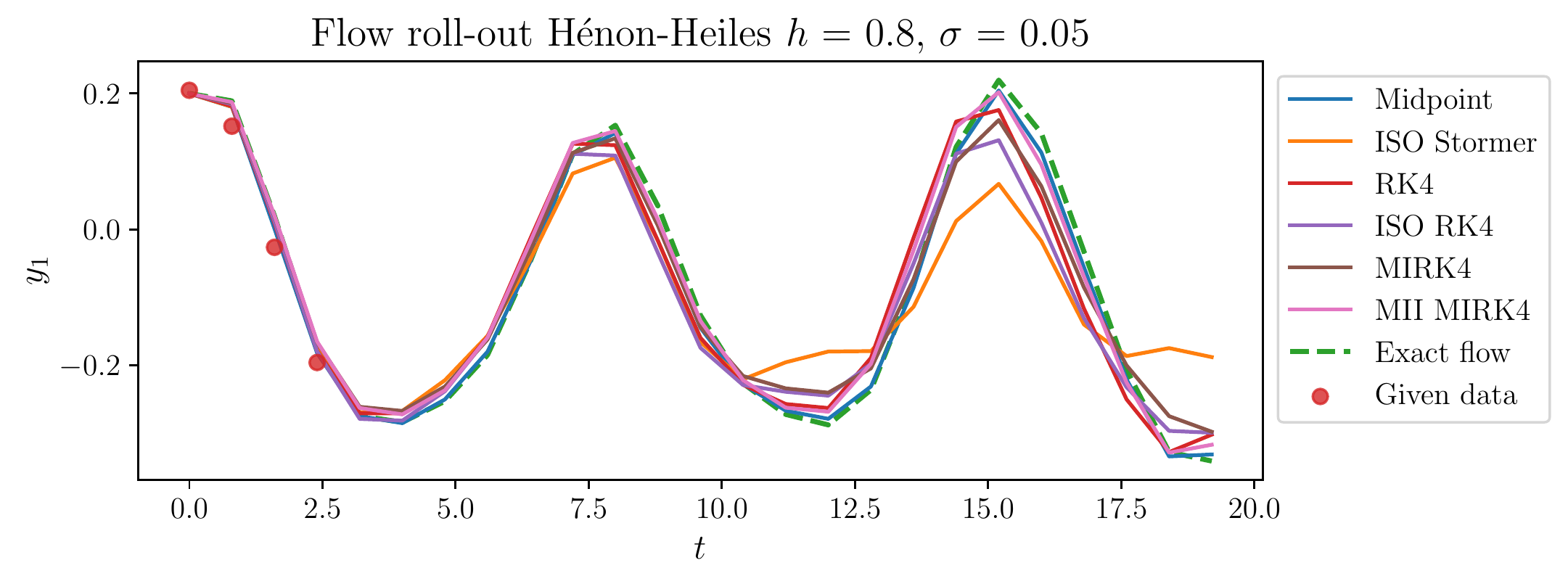}
    \includegraphics[trim=0 0 116 0,clip,width=0.45\textwidth]{plots/roll_out/roll-out_henon-heiles_j=1_sigma=0.05_dt=0.1.pdf}

    \includegraphics[trim=515 170 10 35,clip,width=0.1\textwidth]{plots/roll_out/roll-out_henon-heiles_j=1_sigma=0.05_dt=0.1.pdf}
    \includegraphics[trim=515 138 10 65,clip,width=0.1\textwidth]{plots/roll_out/roll-out_henon-heiles_j=1_sigma=0.05_dt=0.1.pdf}
    \includegraphics[trim=515 102 10 95,clip,width=0.1\textwidth]{plots/roll_out/roll-out_henon-heiles_j=1_sigma=0.05_dt=0.1.pdf}
    \includegraphics[trim=515 135 10 65,clip,width=0.1\textwidth]{plots/roll_out/roll-out-label.pdf}
  \caption{Roll-out in time obtained by integrating over the learned vector fields when training on data from the Fermi--Pasta--Ulam--Tsingou and Hénon--Heiles Hamiltonian.}
    \label{fig:roll_out_time}
\end{figure}

\section{More on numerical integration}\label{more_on_num_int}

\subsection{Runge--Kutta methods}
A general Runge--Kutta method for an autonomous system with $s$ stages is a one-step numerical integrator given by
\begin{equation}
  \begin{aligned}
      y_{n+1} &= y_n + h\sum_{j=1}^s b_i k_i,  \\ 
    k_i &= f\big (y_n + h\sum_{j=1}^s a_{ij} k_j \big ), \quad \quad i = 1,\dots,s.
    \label{rk_stage_form}
  \end{aligned}
\end{equation}
A concrete method is determined by specifying the coefficient matrix $A \in \R^{s\times s}$ and the vector $b \in \R^s$, and there are conditions for symplecticity and order associated with these \cite{hairer93sodI}. The conditions for order $p=1$ require that the coefficient $c\in\R^{s}$ is determined by $c_i = \sum_{j=1}^s a_{ij}$. A method could be compactly represented by a \textit{Butcher tableau} which structures the coefficients the following way
\begin{equation*}
  \def\arraystretch{1.3}
  \begin{tabular}{c|c}
         $c$ & $A$       \\
            \hline
              & $b^T$      \\
  \end{tabular}
\end{equation*}
The two symplectic and symmetric Gauss-Legendre methods (found e.g. in \cite{hairer2006geometric}) with order $p=4,6$ and denoted as GL4 and GL6 in Table \ref{integrator table} are presented in below:

  \begin{figure}[htb]
  \raggedleft
  \begin{equation*}
    \def\arraystretch{1.4}
    \begin{array}{c|cc}
            \hlf - \frac{\sqrt{3}}{6}  & \frac{1}{4} & \frac{1}{4} -  \frac{\sqrt{3}}{6}   \\
            \hlf + \frac{\sqrt{3}}{6}  & \frac{1}{4} + \frac{\sqrt{3}}{6}  & \frac{1}{4}    \\
              \hline
                & \hlf & \hlf     \\
    \end{array}
    \quad \quad
    \def\arraystretch{1.4}
    \begin{array}{c|ccc}
            \hlf - \frac{\sqrt{15}}{10}  & \frac{5}{36} & \frac{2}{9} -  \frac{\sqrt{15}}{15} & \frac{5}{36} -  \frac{\sqrt{15}}{30}   \\
            \hlf  &  \frac{5}{36} +  \frac{\sqrt{15}}{24} & \frac{2}{9} & \frac{5}{36} -  \frac{\sqrt{15}}{24}    \\
            \hlf + \frac{\sqrt{15}}{10}  & \frac{2}{9} +  \frac{\sqrt{15}}{15} & \frac{5}{36} +  \frac{\sqrt{15}}{30} & \frac{5}{36}   \\
            \hline
                & \frac{5}{18} & \frac{4}{9} & \frac{5}{18}      \\
    \end{array}
  \end{equation*}
\label{gl46}
\end{figure}

 \begin{table}[htb]
 \caption{Properties of RK methods. \textit{Symm.} is short for symmetric and \textit{sympl.} for symplectic, and inv. for inverse.}
 \label{integrator table}
 \vskip 0.15in
 \begin{center}
 \begin{small}
 \begin{tabular}{l | l  l l  l l l l}
 \toprule
    Integration method & Name in figures   & Order ($p$) & Stages ($s$) & Symm. & Sympl. & Inv. explicit & Explicit\\ 
 \midrule
     Explicit Euler & E. Euler & 1 & 1 &no &no &yes & yes \\
     Implicit Euler & I. Euler & 1 & 1 &no &no &yes & no\\
     Runge--Kutta 4 & RK4 & 4 & 4 &no &no &yes & yes\\
    Implicit midpoint & Midpoint & 2 & 1 &yes &yes &yes & no\\
    MIRK3 & MIRK3 & 3 & 2 &no &no&yes& no\\
     MIRK4 & MIRK4 & 4 & 3 &yes &no&yes& no\\
     MIRK5 & MIRK5 & 5 & 4 &no &no&yes& no\\
     MIRK6 & MIRK6 & 6 & 5 &yes &no&yes& no\\
    Gauss Legendre 4  & GL4 & 4 & 2 &yes &yes&no& no\\
    Gauss Legendre 6  & GL6  & 6 & 4 & yes &yes&no& no\\
 \bottomrule
 \end{tabular}
 \end{small}
 \end{center}
 \vskip -0.1in
 \end{table}

\subsection{Mono-Implicit Runge--Kutta methods}\label{more_on_mirk}

The MIRK methods are specified by a coefficient vector $b\in \R^s$, $v\in\R^s$ in addition to the strictly lower triangular matrix $D\in\R^{s\times s}$ and could be represented by the an extended Butcher tableau in the following manner
\begin{equation*}
  \def\arraystretch{1.3}
  \begin{tabular}{c|c|c}
         $c$ & $v$ & $D$       \\
            \hline
              & & $b^T$
  \end{tabular}
\end{equation*}
In \cite{burrage1994order} it is proved that the maximum order of an $s$-stage MIRK method is $p = s+1$ and several methods with stages $s \leq 5$ are presented. Below, we specify the MIRK methods used in the numerical experiments in addition to presenting their extended Butcher tableau in Figure \ref{example_mirk}.

\begin{itemize}
    \item Midpoint: The symmetric and symplectic MIRK method where $(s,p) = (1,2)$ is equivalent to the midpoint method.
    \item MIRK3: The method $(s,p)=(2,3)$ found by choosing $c_1 = 1$ in \cite{burrage1994order}.
    \item MIRK4: The method $(s,p)=(3,4)$ with $x_{31} = \frac{1}{8}$ in \cite{muir_optimal_nodate} and is first presented in in \cite{vanBokhoven1980efficient, Cash1980high}.
    \item MIRK5: The method $(s,p)=(4,5)$ presented in \cite{burrage1994order} choosing $c_2=0$ and $c_3=\frac{3}{2}$. It should be noted that as long as $c_3 > 1$ the method is A-stable, however the particular choice of $c_3=\frac{3}{2}$ is arbitrary.
\item MIRK6: The method $(s,p)=(5,6)$ presented in \cite{muir_optimal_nodate}, which is the $s=5$ stage scheme in \cite{burrage1994order} choosing $c_3 = \frac{1}{2} - \frac{\sqrt{21}}{14}$. According to \cite{muir_optimal_nodate}, this method is an improvement over earlier schemes on the same form which used $c_3 = \frac{1}{4}$.
\end{itemize}

\begin{figure}
  \raggedleft
\begin{align*}
  \def\arraystretch{1.2}
  \begin{array}{c|c|cc}
    1& 1 &0  & 0 \\
    \frac{1}{3} &  \frac{5}{9} &-\frac{2}{9} & 0 \\
    \hline
    & & \frac{3}{4} &\frac{1}{4} 
  \end{array}
\quad \quad
  \def\arraystretch{1.2}
  \begin{array}{c|c|ccc}
    0& 0 &0  & 0 & 0 \\
    1 &  1 &0 & 0 &0 \\
    \hlf& \hlf& \frac{1}{8} &  -\frac{1}{8} & 0 \\
    \hline
    & & \frac{1}{6} &\frac{1}{6}&\frac{2}{3} 
  \end{array}
  \quad \quad
  \def\arraystretch{1.2}
  \begin{array}{c|c|cccc}
0&0&0 & 0 & 0 & 0\\
1&1&0 & 0 & 0 & 0\\
\frac{3}{2}&0&\frac{3}{8} & \frac{9}{8} & 0 & 0\\
\frac{9}{20}&\frac{40257}{80000}&\frac{16929}{160000} & - \frac{5643}{32000} & \frac{693}{40000} & 0\\
    \hline
&&\frac{23}{162}& \frac{5}{22}&- \frac{2}{189}& \frac{4000}{6237}
\end{array}
\end{align*}
\begin{align*}
  \def\arraystretch{1.4}
\begin{array}{c|c|ccccc}
  0&0&0 & 0 & 0 & 0 & 0\\
  1&1&0 & 0 & 0 & 0 & 0\\
\hlf - \sqd &\hlf - \sqb&\fa +\sqa & -\fa +\sqa & 0 & 0 & 0\\
\hlf + \sqd&\hlf + \sqb&\fa -\sqa & -\fa -\sqa & 0 & 0 & 0\\
\hlf&\hlf&- \frac{5}{128} & \frac{5}{128} & \sqb& -  \sqb& 0\\
  \hline
  &&\frac{1}{20}&\frac{1}{20}&\frac{49}{180}&\frac{49}{180}&\frac{16}{45}
\end{array}
\end{align*}
\caption{Extended Butcher tableau of MIRK methods with stage and order $(s,p) = (2,3),(3,4),(4,5),(5,6)$.}
\label{example_mirk}
\end{figure}

\subsection{Symmetric methods:}
The exact flow of an ODE satisfies the following property known as (time) symmetry: \[y(t_0)=\varphi_{h,f}^{-1}(y(t_0 +h))=\varphi_{-h,f}(y(t_0 + h)),\]
where the superscript ``$-1$" denotes the inverse map. 
This is a desirable property also for the numerical approximation. A numerical integration method $\Phi_{h,f}$ is called symmetric if
\begin{equation}
    \label{symmetry}
\Phi_{h,f} = \Phi_{-h,f}^{-1}.
\end{equation}
Symmetric numerical methods have the following properties \cite{chartier_symmetric_2015}:
\begin{enumerate}
  \item A symmetric integrator preserves the (time) symmetry of the exact flow.
  \item The order $p$ of a symmetric method is necessarily even.
  \item Solutions of Hamiltonian systems satisfy the following reflection symmetry: if $\big(q(t),p(t)\big)$ solves the Hamiltonian ODE, then $\big(q(-t),-p(-t)\big)$ is also a solution, with $y(t) = [q(t),p(t)]^T$. Numerical solutions $(q_n,p_n)$ obtained from a symmetric Runge--Kutta method satisfy the same reflection symmetry \cite{leimkuhler_reich_2005}.
\end{enumerate}
 A Runge--Kutta method is symmetric if and only if 
\begin{align}
\label{symmetric runge kutta}
  PA + AP - \one b^T &= 0,\\
  b &= Pb,
\end{align}
where $\one := [1,\dots,1]^T \in\R^{s}$ and $[p]_{ij} = \delta_{i,s+1-j}$ \cite{chartier_symmetric_2015}. That is, $P$ is the reflection of the identity matrix over the first axis. Inserting the definition of a MIRK method from \eqref{mirk-stages}, we get
\begin{align*}
  PD + DP + (Pv + v - \one)b^T &= 0\\
  b &= Pb.
\end{align*}
Symmetric MIRK methods of order $p = 2,4,6$ are presented in \cite{Cash1982high, muir_optimal_nodate} and specific examples are found in Figure \ref{example_mirk}.

\section{Details on the inverse injection in MII}\label{appendix:detail_of_mii}
Assume we are deriving the MII following the example in Equation \eqref{averaging_trajectories} using the implicit midpoint method, where
 \[y_{n+1} =  y_n + hf\big( \frac{y_n +  y_{n+1}}{2} \big) = y_n + h\Psi_{n,n+1}. \]
We thus find that the second term in \eqref{averaging_trajectories}, the composition of two steps starting in $\tilde y_0$ could be approximated by
  \begin{align*}
   \hat y_2 &=\Phi_{h,f} \circ \Phi_{h,f}(\tilde y_{0})\\
   &= \Phi_{h,f}(\tilde y_{0}) + hf\bigg (\frac{ \Phi_{h,f}(\tilde y_{0}) + \hat y_2}{2} \bigg )\\
    &\approx \Phi_{h,f}(\tilde y_{0}) + hf\bigg ( \frac{\tilde y_1 + \tilde y_2}{2}\bigg  )\\
    &= \tilde y_0 + hf\bigg ( \frac{\tilde y_0 + \Phi_{h,f}(\tilde y_{0})}{2} \bigg ) + h\Psi_{1,2}\\
    &\approx \tilde y_0 + hf\bigg ( \frac{\tilde y_0 + \tilde y_1 }{2} \bigg ) + h\Psi_{1,2}\\
    &= \tilde y_0 + h\Psi_{0,1} + h\Psi_{1,2}.\\
  \end{align*}
where the approximation $\approx$ is obtained by the substitution $\tilde y_2  \rightarrow \hat y_2 $ and  $\tilde y_1 \rightarrow \Phi_{h,f}(\tilde y_{0})$. The same procedure (repeatedly using the inverse injection) is generalized over longer trajectories and used to arrive at the MII method in Definition \ref{mii_equation}.

\section{Details on neural network training}\label{appendix:details_nn_train}

The experiments were performed on a Apple M1 Pro chip with double precision. The PyTorch L-BFGS \cite{pytorch} algorithm is run with the following parameters:
\begin{itemize}
    \item History size: $120$.
    \item Gradient tolerance: $10^{-9}$.
    \item Termination tolerance on parameter changes: $10^{-9}$.
    \item Line search: Strong Wolfe.
\end{itemize}

Both MII and ISO work better when $f_{\theta}$ has been pre-trained to be a reasonable approximation of the underlying vector field $f$. Thus, for both MII and ISO training is run $10$ epochs on the one-step method before training additional $10$ epochs with MII (MII MIRK4) and ISO (ISO Störmer and ISO RK4). The ISO procedure (searching for the optimal initial value $\hat y_0$) utilizes the L-BFGS optimization algorithm from the SciPy library \cite{2020SciPyNMeth} with gradient tolerance of $10^{-6}$ and the maximum number of iterations limited to $10$.

\section{Proof of Theorem \ref{max_ord_symp}}\label{proof_max_order}

\begin{proof}
  As stated in Equation \eqref{symplectic_condition_rk} Runge--Kutta method as given by Equation \eqref{rk_stage_form} is symplectic if and only if
\begin{equation*}
  b_i a_{ij} + b_j a_{ji} - b_i b_j = 0.
\end{equation*}
Inserting the particular form for the MIRK coefficients, $a_{ij} = d_{ij} + v_ib_j$, we get 
\begin{align}
  b_i(v_i b_j + d_{ij}) + b_j(v_j b_i + d_{ji}) - b_ib_j &= 0 \nonumber \\
  b_i d_{ij} + b_j d_{ji} + b_i b_j (v_j + v_i - 1) &= 0.
  \label{mirk_proof_1}
\end{align}
As $D$ is strictly lower triangular eiter $d_{ji} = 0$ or $d_{ij} = 0$, which for Equation \eqref{mirk_proof_1} implies that
\begin{align*}\begin{cases}
 b_i d_{ij} + b_i b_j (v_j + v_i - 1) = 0 &\quad\quad  \text{if} \; i \neq j\\
´ b_i^2 (2v_i - 1) = 0 &\quad\quad \text{if} \; i = j
  \end{cases}
\end{align*}
Requiring $d_{ij}, b_i$ and $v_i$ to satisfy the symplecticity condition yields the following restriction
\begin{equation}
\begin{aligned}
\begin{cases}
b_i d_{ij} + b_i b_j (v_j + v_i - 1) = 0  &\quad \text{if} \; i \neq j,\\
b_i = 0 \;\; \text{or} \;\; v_i = \frac{1}{2} & \quad \text{if} \; i = j.
\label{symp_mirk_req}
  \end{cases}
\end{aligned}
\end{equation}
Without loss of generality, we assume that the $m$ first entries of $b\in\R^s$ are zero. Enforcing the conditions of Equation \eqref{symp_mirk_req} on $v\in\R^s$ we get for $1 \leq m\leq s$
\begin{align*}
  b &= [0,\dots,0,b_{m+1},\dots,b_s]^T,\\
  v &= [v_1,\dots,v_m,\hlf,\dots,\hlf]^T.
\end{align*}
In total, this gives the following constraints for $v,b,D$:
\begin{equation*}
    \begin{array}{c|c|cc}
    
      \;    & b_j = 0 & b_j \neq 0 \\
      \hline
              &                       &           \\
      b_i = 0 &         d_{ij} \in \R & d_{ij} \in \R \\
              &         v_i, v_j \in \R  & v_i,v_j \in \R \\
              &                       &           \\
      \hline
      &                       &           \\
      b_i \neq 0&          d_{ij} = 0 & d_{ij} =0 \\
      &             v_i,v_j \in \R   & v_i, v_j = \frac{1}{2}  \\
                &                       &           \\
    \end{array}
    \end{equation*}

Which for the Runge--Kutta method $A = D + vb^T$ gives a (RK) Butcher tableau of the form
  \begin{equation*}
    \def\arraystretch{1.2}
      \begin{array}{c|cccc|ccc}
      
                & 0 & 0 &  \hdots & 0    & v_1 b_{m+1}  &  \hdots & v_1 b_s        \\
                & d_{21} & 0  & \hdots & 0    &  &     &         \\
                & d_{31} & d_{32} &  & 0    & \vdots &      & \vdots        \\
                & \vdots &    & \ddots &     &   &   &       \\
                & d_{m,1} & \hdots  & d_{m,m-1}& 0    & v_{m} b_{m+1}  &  \hdots & v_m b_s        \\
                \hline 
                & 0 & \hdots  & \hdots   & 0    & \hlf b_{m+1}  &  \hdots & \hlf  b_s        \\
                & \vdots &   &    & \vdots    & \vdots &      & \vdots        \\
                & 0 & \hdots & \hdots &    0    &\hlf  b_{m+1} &  \hdots &\hlf  b_s        \\
                \hline
                & 0 & \hdots & \hdots &  0    &  b_{m+1}  &  \hdots &  b_s        \\
      \end{array}
      \end{equation*}
      Since the lower left submatrix is the zero matrix, this leaves the stages $k_{m+1},\dots,k_s$ unconnected to the first $m$ stages. In addition, as $b_i = 0$ for $i=1,\dots,m$, these stages are not included in the computation of the final integration step. The method is thus reducible to the lower right submatrix of $A$ and $b_{m+1},\dots,b_s$.
The reduced method is thus in general given by the following stage-values
\begin{align*}
    k_i = f\big(y_n + \frac{h}{2}\sum_{j}^s b_jk_j\big).
\end{align*}
  It is trivial to check that if $\sum_{i}^s b_i = 1$ the method satisfies order conditions up to order $p=2$, which could be found in \cite{hairer2006geometric} to be
  \begin{align*}
      \sum_{i} b_i &= 1, \quad \quad \text{and} \quad \quad
      \sum_{i,j} b_i a_{ij} = \frac{1}{2}.
  \end{align*}
  However, the method fails to satisfy the first of the two conditions required for order $p=3$, since
  \begin{align*}
      \sum_{i,j,k} b_i a_{ij} a_{ik} = \frac{1}{4} \sum_{i,j,k} b_i b_j b_k = \frac{1}{4} \neq \frac{1}{3}.
  \end{align*}
  Hence, the maximum order of a symplectic MIRK method is $p=2$.
\end{proof}
As a remark, it should be noted that the $s=1$ stage, symplectic MIRK method found by setting $b_1 = 1$, $v_1 = \frac{1}{2}$ and $d_{11} = 0$ is simply the midpoint method 
   $ y_{n+1} = y_n + hk_1$
 with   $k_1 = f(\frac{y_n + y_{n+1}}{2})$.

\section{Proof of Theorem \ref{thm:propagation of noise mirk}}\label{proof:propagation of noise mirk}
\begin{proof}
  Let $s_i(y_n,y_{n+1}) := y_n + v_i(y_{n+1} - y_n)$ and $\tilde y_n$ be noisy data \eqref{noisy_obs}. Observe that we can obtain the following approximation to the MIRK stages \eqref{mirk-stages} by
  \begin{align}
    k_i &= f\bigg(\tilde y_n + v_i(\tilde y_{n+1} - \tilde y_n) + h\sum_{j=1}^s d_{ij} k_j \bigg)\nonumber \\
    &= f\bigg (s_i(y_{n},y_{n+1}) + s_i(\delta_{n},\delta_{n+1}) + \mathcal O(h) \bigg )\nonumber  \\
    &= f(s_i(y_{n},y_{n+1})) + f'\big(s_i(y_{n},y_{n+1})\big)s_i(\delta_{n},\delta_{n+1}) + \mathcal O(\|s(\delta_n,\delta_{n+1}) \|^2) + \mathcal O(h)\nonumber \\
    &= f(y_n) + f'(y_n)s_i(\delta_{n},\delta_{n+1}) + \mathcal O(\|s(\delta_n,\delta_{n+1}) \|^2) + \mathcal O(h).
    \label{mirk_stage_approx}
  \end{align}
Where in the final equality we expand $y_{n+1} = y_n + hf(y_n) + \mathcal O(h^2)$ to find
\begin{align*}
  f(s_i(y_n,y_{n+1})) &= f(y_n + v_i(y_{n+1} - y_n))\\
  &= f(y_n + hv_i f(y_n) + \mathcal O(h^2))\\
  &= f(y_n) + \mathcal O(h).
\end{align*}
And similarly for $f'(s_i(y_n,y_{n+1}))$. In total, this means that the next MIRK-step could be approximated by
\begin{align}
  \label{mii_approx}
  y_{n+1} &= \tilde y_n + h\sum_{i=1}^s b_i k_i \nonumber \\
  &= \tilde y_n +  h\sum_{i=1}^s b_i \bigg ( f(y_{n}) + f'(y_n)s_i(\delta_{n},\delta_{n+1})\bigg ) + \mathcal O(h\|s(\delta_n,\delta_{n+1})^2 \|) + \mathcal O(h^2).
\end{align}
First note, that if $x$ is a multivariate normally distributed random variable $x \sim \mathcal N(0,\Sigma)$ then for a matrix $G \in \R^{n\times n}$ the variance of the linear transformation is given by $\V[Gx] := \text{Cov}[Gx,Gx] =  G\Sigma G^T $. Now, using the approximation in Equation \eqref{mii_approx}, we find the variance of the optimization target $\mathcal T^{\text{OS}}_{n+1} =\tilde y_{n+1} - \Phi_{h,f}(\tilde y_{n}, \tilde y_{n+1})$ by
\begin{align*}
  \V \bigg [ \tilde y_{n+1} - \tilde y_n - h\sum_{i=1}^s b_i k_i \bigg] \approx&   \V \bigg [ \delta_{n+1} - \delta_n - h\sum_{i=1}^s b_i \bigg ( f'(y_n)s_i(\delta_{n},\delta_{n+1})\bigg ) \bigg] \\
  =& \;  \V \bigg [ \bigg (I - hb^Tv f'(y_n)\bigg )\delta_{n+1} - \bigg (I + h b^T(\one - v) f'(y_n) \bigg )\delta_n \bigg] \\
  =& \; \sigma^2   \bigg (I - hb^Tv f'(y_n)\bigg ) \bigg (I - hb^Tv f'(y_n)\bigg )^T \\
  &+  \sigma^2  \bigg (I + hb^T(\one - v) f'(y_n) \bigg ) \bigg (I + hb^T(\one - v) f'(y_n) \bigg )^T \\
  =& \; \sigma^2 \bigg [2I + hb^T(\one - 2v)\big(f'(y_n)\! +\! f'(y_n)^T\big) \\&+ h^2\underbrace{\bigg ((b^Tv)^2 + (b^T(\one - v))^2\bigg)f'(y_n)f'(y_n)^T}_{:= Q^{\text{OS}}}  \bigg ]\\
  =& \; \sigma^2 \bigg [2I + hb^T(\one - 2v)\big(f'(y_n) \! + \! f'(y_n)^T\big) + h^2 Q^{\text{OS}}  \bigg ].\\
\end{align*} 
Here, $\one := [1,\dots,1]^T\in\R^s$. This is the variance estimate we wanted to find for MIRK methods used as one-step integration schemes.
Similarly, considering a point computed by the mean inverse integrator  $\overline y_n$, we find, using the stage approximation by Equation \eqref{mirk_stage_approx} that
\begin{align*}
  \overline y_n &= \frac{1}{N}\sum_{\substack{j = 0 \\j\neq n} }^{N} \tilde y_j +  \frac{h}{N}\sum_{j=0}^{N-1} w_{n,j}  \sum_{l=1}^s b_l k_l\\
  &\approx \frac{1}{N}\sum_{\substack{j = 0 \\j\neq n} }^{N} \tilde y_j +  \frac{h}{N}\sum_{j=0}^{N-1} w_{n,j}   \sum_{l=1}^s b_l \bigg ( f(y_j) + f'(y_j)s_l(\delta_{j},\delta_{j+1})\bigg )
\end{align*}
Where we note that $w_{n,j} := [W]_{nj}$, from Definition \ref{mii_equation} of the MII. Let $\overline y_n := [\overline Y]_n$. Computing the variance of the optimization target $\mathcal T^{\text{MII}}_{i} = \tilde y_n - \overline y_n$ we find, by introducing $\overline P_{nj}$ to simplify notation, that
\begin{align*}
  \V \bigg [ \tilde y_{n} - \overline y_{n} \bigg] \approx&   \V \bigg [ \delta_{n} - 
    \frac{1}{N}\sum_{\substack{j = 0 \\j\neq n} }^{N} \delta_j -  \frac{h}{N}\sum_{j=0}^{N-1} w_{n,j}  f'(y_j) \bigg ( b^T(\one - v)\delta_j + b^Tv \delta_{j+1}  \bigg)\bigg ] \\
    =&  \V \bigg [ \frac{1}{N} \sum_{\substack{j = 0 \\j\neq n}}^N \bigg (I + h\underbrace{f'(y_j)\bigg ( \tilde w_{n,j} b^T(\one - v) + \tilde w_{n,j-1}b^Tv  \bigg)}_{:=\overline  P_{nj}} \bigg   )\delta_j \bigg ]\\
    &+ \V \bigg [\bigg ( I -  \frac{h}{N} \underbrace{ f'(y_n)\bigg ( \tilde w_{n,n} b^T(\one - v) + \tilde w_{n,n-1}b^Tv  \bigg   )}_{=\overline P_{nn}} \bigg ) \delta_n \bigg ]\\
    =& \V \bigg [ \frac{1}{N} \sum_{\substack{j = 0 \\j\neq n}}^N \bigg (I + h\overline P_{nj} \bigg   )\delta_j \bigg ]
    +  \V \bigg [\bigg ( I -  \frac{h}{N} \overline P_{nn}\bigg ) \delta_n \bigg ]
    \\
    =&  \frac{\sigma^2 }{N^2} \sum_{\substack{j = 0 \\j\neq n}}^N  \bigg (  I + h\overline  P_{nj}  \bigg ) \bigg (I + h\overline P_{nj} \bigg )^T
    + \sigma^2 \bigg ( I -  \frac{h}{N} \overline P_{nn}\bigg )\bigg( I -  \frac{h}{N} \overline P_{nn}\bigg )^T.
  \end{align*}
 In the second line, $\tilde w_{n,j}$ is introduced which is elements of a matrix $\tilde W = [ 0 | w_1 | w_2 | \dots | w_N | 0 ] \in \R^{N \times N+1}$, or in other words the matrix you obtain by padding $W$ right and left with a column of zeros.
  Expanding the terms and introducing matrices $P_{nj}$ and $Q^{\text{MII}}$ we finally find
    \begin{align*}
      \V \bigg [ \tilde y_{n} - \overline y_{n} \bigg] \approx& \frac{\sigma^2}{N} \bigg[ (1 + N)I + h\underbrace{(\overline P_{nn} + \overline P_{nn}^T)}_{= P_{nn}} +  \frac{h}{N}\sum_{\substack{j = 0 \\j\neq n}}^s\underbrace{( \overline P_{nj} + \overline P_{nj}^T)}_{:= P_{nj}} + \frac{h^2}{N} \underbrace{   \sum_{j = 0}^s \overline P_{nj}\overline P_{nj}^T}_{:= Q^{\text{MII}}} \bigg]\\
    =& \frac{\sigma^2}{N} \bigg[ (1 + N)I + hP_{nn} +  \frac{h}{N}\sum_{\substack{j = 0 \\j\neq n}}^s P_{nj} + \frac{h^2}{N} Q^{\text{MII}} \bigg].
\end{align*} 
Since for a symmetric matrix $A$ we have that the spectral radii $\rho$ (largest absolute value of eigenvalues) could be found by $\rho(A) = \|A\|_2$, we find for both variance approximations (covariance matrix is always symmetric) that
    \begin{align*}
      \rho \bigg( \V \big [ \tilde y_{n} - \overline y_{n} \big]\bigg ) &\approx 
    \frac{\sigma^2}{N} \bigg \|  (1 + N)I + hP_{nn} +  \frac{h}{N}\sum_{\substack{j = 0 \\j\neq n}}^s P_{nj} + \frac{h^2}{N} Q^{\text{MII}}  \bigg \|_2\\
      \rho \bigg(  \V \big [ \tilde y_{n+1} - \Phi_{h,f}(\tilde y_n,\tilde y_{n+1}) \big] \bigg )&\approx
        \sigma^2 \bigg \|2I + hb^T(\one - 2v)\big(f'(y_n) \! + \! f'(y_n)^T\big) + h^2 Q^{\text{OS}} \bigg \|_2 .\\
\end{align*}

Finally, we note that:

\begin{equation}
\begin{aligned}
  Q^{\text{OS}} &:=\bigg ((b^Tv)^2 + (b^T(\one - v))^2\bigg)f'(y_n)f'(y_n)^T\\
 \overline P_{nj} &:= f'(y_j)\bigg ( \tilde w_{n,j} b^T(\one - v) + \tilde w_{n,j-1}b^Tv  \bigg   )\\
 P_{nj} &:= \overline P_{nj} + \overline P_{nj}^T\\
 Q^{\text{MII}} &:=   \sum_{j = 0 }^s \overline P_{nj}\overline P_{nj}^T.
\end{aligned}
   \label{variance_matrices}
\end{equation}

\end{proof}

\section{Higher-order inverse-explicit invariant-preserving symmetric non-partitioned integrators}\label{appendix:higher_order_inv_explicit}
We define invariant-preserving integrators as methods that preserve the Hamiltonian or other invariants of the exact solution, either exactly up to machine precision or within a bound, like symplectic methods. Although we argue in this paper that symplecticity is a less important property when learning Hamiltonian systems from data than for integration of a known system, we do not mean to suggest that invariant-preserving integrators may not be beneficial to some extent and have important qualities in the inverse problem also. However, we urge anyone who seeks to use invariant-preserving methods to also consider the order of the method and whether it is a symmetric inverse-explicit method. Although the maximum order of a symplectic inverse-explicit \textit{Runge--Kutta} method is two, there exist higher-order inverse-explicit invariant-preserving integrators that are not Runge--Kutta methods.

Note that \textit{partitioned} Runge--Kutta (PRK) methods is an extension that does not belong to the class of Runge--Kutta methods. This is important to clarify since there exist PRK methods that are symmetric and explicit for separable systems. This marks a distinction from non-partitioned RK methods: these cannot be symmetric and explicit in general \cite{hairer2006geometric}. Several papers suggest using symplectic PRK methods for learning Hamiltonian systems \cite{Chen2020Symplectic, dipietro2020sparse, desai2021variational}, but these methods, although symmetric, only depend on one point to approximate the right-hand side of each integration step, and thus do not average out any noise.

\subsection{Symplectic elementary differential Runge--Kutta methods}
Chartier et al.\ showed in \cite{Chartier07} that an integrator can be applied to a modified vector field in such a way that it yields a higher-order approximation of the original vector field while inheriting the geometric properties of the given integrator. As an example, they present the fourth-order modified implicit midpoint method
\begin{equation}\label{eq:imp4}
\frac{y_{n+1}-y_n}{h} = f(\bar{y}) + \frac{h}{12} \big(-Df(\bar{y})Df(\bar{y})f(\bar{y}) + \frac{1}{2}D^2 f(\bar{y}) f(\bar{y}) f(\bar{y}) \big),
\end{equation}
where $\bar{y} = (y_n+y_{n+1})/2$. This is an example of an elementary differential Runge--Kutta (EDRK) method \cite{uria1995metodos}, which relies on the calculation of (multi-order) derivatives of the vector field $f$, denoted here as $D^p f$ for order $p$. Automatic differentiation can be utilized also to get higher-order derivatives, and we note that $f$, $Df$ and $D^2f$ each only have to be evaluated once for each training step since they are only evaluated at the one point $\bar{y}$. A sixth-order modification of the implicit midpoint method is also presented in \cite{Chartier07}, but that requires the calculation of up to fourth-order derivatives and might be considered prohibitively expensive.

\subsection{Discrete gradient methods}
Discrete gradient methods are a class of integrators that can preserve an invariant, e.g.\ the Hamiltonian, exactly \cite{mclachlan1999geometric}. This is in contrast to symplectic methods, which only preserve a perturbation of the invariant exactly and the exact invariant within some bound. We remark that no method can be both symplectic and exactly invariant-preserving in general \cite{Zhong1988Lie}. Discrete gradient methods are defined strictly for invariant-preserving ODEs, which can be written on the form
\begin{equation}\label{eq:invariantode}
\dot{y} = S(y) \nabla H(y),
\end{equation}
for some skew-symmetric matrix $S(y)$ \cite{mclachlan1999geometric}. Then a discrete gradient is a function $\overline{\nabla} H : \mathbb{R}^d \times \mathbb{R}^d \rightarrow \mathbb{R}$ satisfying
\begin{equation*}
\overline{\nabla} H(u,v)^T (u-v) =  H(u)-H(v),
\end{equation*}
a discrete analogue to the invariant-preserving property $\dot{H}(y) = \nabla H(y)^T\dot{y} = 0$ of \eqref{eq:invariantode}. A corresponding discrete gradient method is then given by
\begin{equation}\label{eq:dgm}
\frac{y_{n+1}-y_n}{h} = \overline{S}(y_n,y_{n+1},h) \overline{\nabla} H(y_n,y_{n+1}),
\end{equation}
for some approximation $\overline{S}(y_n,y_{n+1},h)$ of $S(y)$ such that $\overline{S}(y,y,0) = S(y)$, where $h$ is the step size in time. A discrete gradient can at most be a second-order approximation of the gradient, but appropriate choices of $\overline{S}$ can yield inverse-explicit integrators of arbitrarily high order \cite{eidnes2022order}. Matsubara et al.\ have developed a discrete version of the automatic differentiation algorithm that makes it possible to efficiently calculate a discrete gradient of neural network functions, and demonstrated its use in training of HNNs \cite{DeepEnergy-BasedModelingofDiscrete-TimePhysics} and for detecting invariants \cite{Matsubara2022finde}. A fourth-order discrete gradient method is suggested for training HNNs in \cite{eidnes2022order}, given a constant $S$ in \eqref{eq:invariantode}. This is the scheme \eqref{eq:dgm} with
\begin{equation*}
\overline{S}(y_n,\cdot,h) = S + \frac{8}{9} h S Q(y_n,z_2) S - \frac{1}{12} h^2 \, SD^2 H(z_1)S D^2 H(z_1)S,
\end{equation*}
with $z_1 = y_n + \frac{1}{2} h f(y_n)$, $z_2 = y_n + \frac{3}{4}h f(z_1)$ and $Q(u,v) := \frac{1}{2} (D_2 \overline{\nabla}H (u,v)^T - D_2 \overline{\nabla}H(u,v)),$ where $D_2 \overline{\nabla}H$ denotes the derivative of $\overline{\nabla}H$ with respect to the second argument, and $D^2H :=D\nabla H$ is the Hessian of $H$. This is not symmetric, so we propose here instead the \textit{fourth-order symmetric invariant-preserving scheme} obtained by
\begin{align*}
\overline{S}(y_n,y_{n+1},h) = & \, S + \frac{h}{2} S \big(Q(y_n,\frac{1}{3}y_n+\frac{2}{3}y_{n+1})-Q(y_{n+1},\frac{2}{3}y_n+\frac{1}{3}y_{n+1})\big) S \\
& \, - \frac{1}{12} (h)^2 \, SD^2 H(\bar{y})S D^2 H(\bar{y})S.
\end{align*}

\subsection{Numerical comparison of fourth-order integrators}

We test four different fourth-order integrators on solving an initial value problem of the double pendulum described in Appendix \ref{test_problems}. We compute an approximation of the error of the solution at each time by comparing to a solution obtained using RK4 with 10 times as many time steps. As seen in the left plot of Figure \ref{fig:integration}, the symmetric methods are clearly superior to the explicit RK4 method, when using the same step size. For integration, the advantage of RK4 is that it is more computationally efficient than the implicit methods, which facilitates taking smaller step sizes. However, as pointed out in Section \ref{sec:inverse_problems}, RK4 does not have this advantage over MIRK methods for the inverse problem.

Furthermore, although the higher-order MIRK methods we suggest to use in this paper are not symplectic and thus lack general energy preservation guarantees, we see from Figure \ref{fig:integration} that they may still preserve the energy within a bound for specific problems. In fact, for the double pendulum problem considered here, the non-symplectic MIRK4 method preserves the energy slightly better than the symplectic MIMP4 scheme up to time $T=500$. The invariant-preserving discrete gradient method preserves the Hamiltonian to machine precision.
\begin{figure}[!htb]
    \centering
    \includegraphics[width=0.47\textwidth]{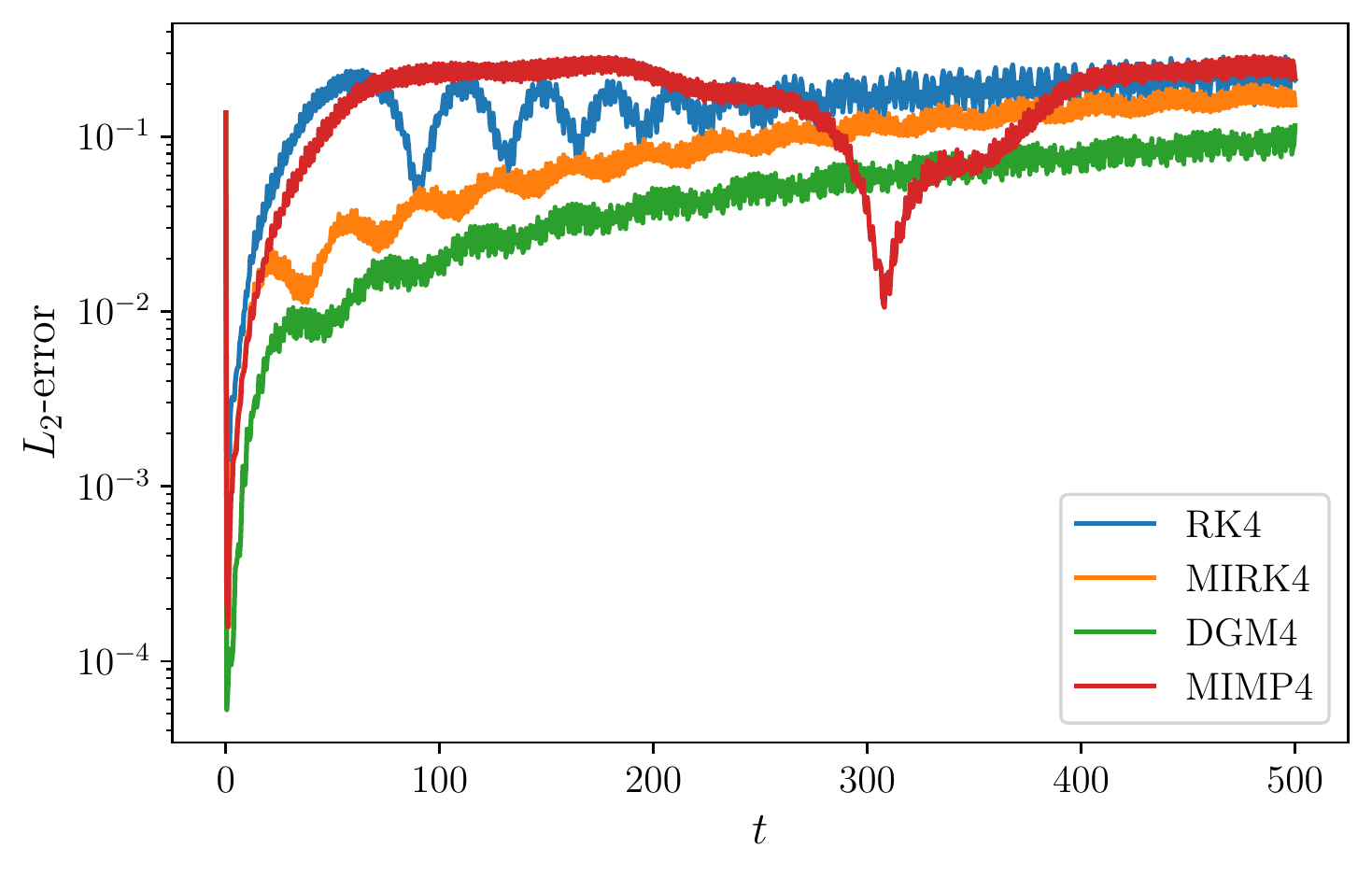}
    \includegraphics[width=0.47\textwidth]{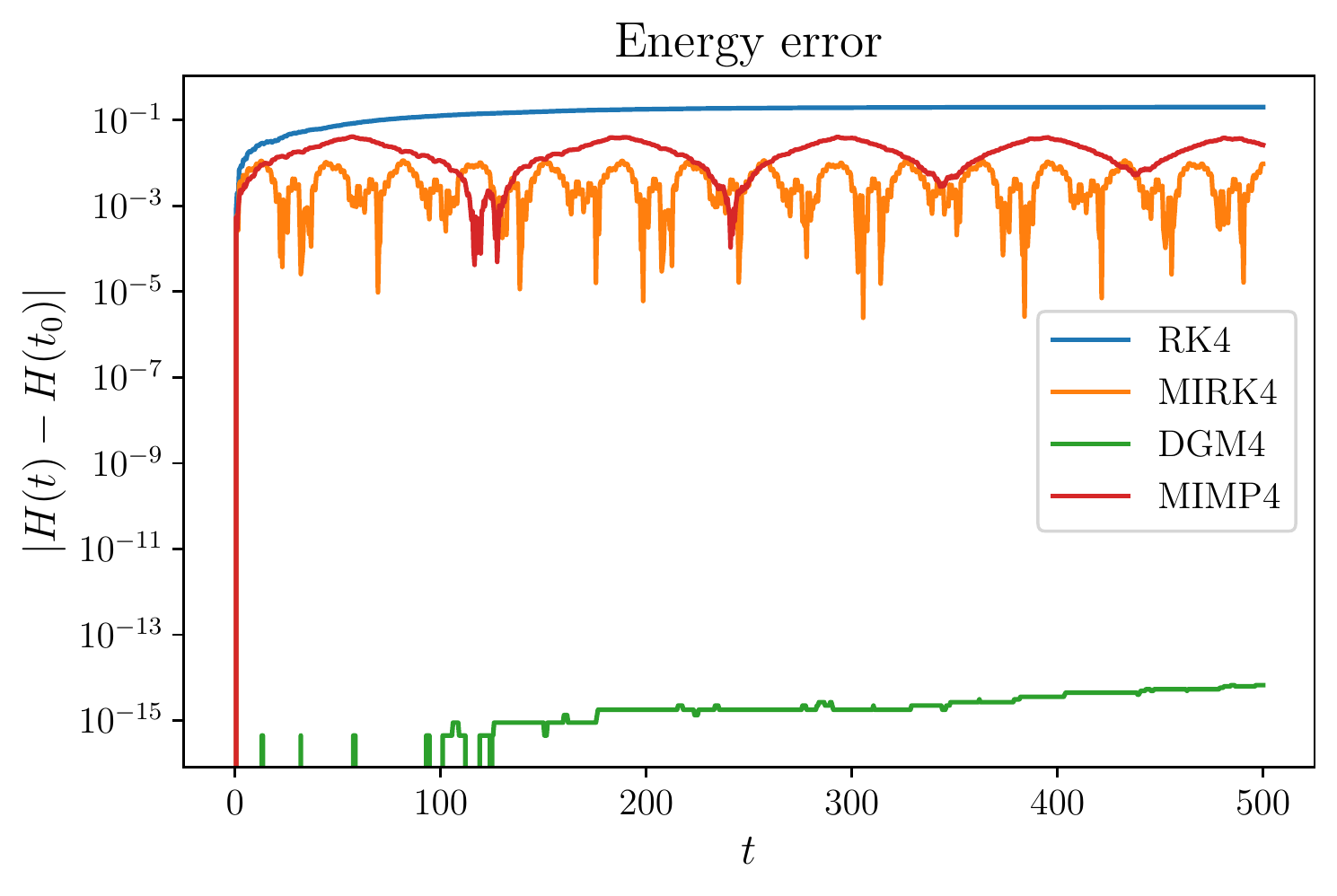}
    \caption{Global error (\textit{left}) and energy error (\textit{right}) of the solution of the double pendulum problem obtained using four different integrators. The initial condition is $y_0=[0.1,0.3,-0.4,0.2]^T$, and the step size for all integrators is $h=\frac{1}{2}$.}
    \label{fig:integration}
\end{figure}

\section{Computational cost}\label{appendix:comp_cost}

The fourth-order MIRK method from Table \ref{example_mirk} (MIRK4) is between twice and thrice as expensive as the implicit midpoint method, depending on the training strategy. That is, if no batching is performed and $f$ is evaluated at all points in the training set at each iteration of the optimization, then the number of function evaluations for a trajectory with $n$ points is $n-1$ for the implicit midpoint method and $2n-1$ for MIRK4. However, if batching is done and function evaluations cannot generally be reused for successive points, the total number of function evaluations at each epoch may increase to $3n-3$ for MIRK4.

In general, the cost of an $s$-stage MIRK method depends on both the training strategy and whether the end points $y$ and $\hat{y}$ are two of the stages. If batching is not done and $y$ and $\hat{y}$ are two of the stages, then computational cost at each epoch is $\mathcal{O}\left(m (n + (s-2)(n-1)) \right)$, where $m$ is the number of trajectories of $n$ points in each. The maximum cost with batching is the same as the cost if $y$ and $\hat{y}$ are not two of the stages: $\mathcal{O}\left(m s (n-1) ) \right)$. This cost is equivalent to that of an explicit $s$-stage RK method.

\end{document}